\documentclass{uai2021} % for initial submission
\usepackage[american]{babel}
\usepackage{natbib} % has a nice set of citation styles and commands
\usepackage{mathtools} % amsmath with fixes and additions
\usepackage{booktabs} % commands to create good-looking tables
\usepackage{tikz} % nice language for creating drawings and diagrams

\usepackage{graphicx}
\usepackage{caption}
\usepackage{algorithm}
\usepackage{algorithmic}
\usepackage{amsmath}
\usepackage{amssymb}
\usepackage{amsthm}
\usepackage{booktabs}       % professional-quality tables
\usepackage{amsfonts}       % blackboard math symbols
\usepackage{nicefrac}       % compact symbols for 1/2, etc.
\usepackage{microtype}      % microtypography
\usepackage{subfig}
\usepackage{bm}
 \usepackage{threeparttable}

\makeatletter
\newcommand{\rmnum}[1]{\romannumeral #1}
\newcommand{\Rmnum}[1]{\expandafter\@slowromancap\romannumeral #1@}
\makeatother
\newtheorem{theorem}{\textbf{Theorem}}
\newtheorem{lemma}{\textbf{Lemma}}

\newtheorem{definition}{\textbf{Definition}}

\newtheorem{remark}{\textbf{Remark}}

\makeatletter
\def\saveenum{\xdef\@savedenum{\the\c@enumi\relax}}
\def\resetenum{\global\c@enumi\@savedenum}
\makeatother

\title{Defending SVMs Against Poisoning Attacks: The Hardness and DBSCAN Approach}

\author[]{Hu Ding} % Lead author
\author[]{Fan Yang}
\author[]{Jiawei Huang}
\affil[]{%
    School of Computer Science and Technology\\
    University of Science and Technology of China\\
    He Fei, China \\
    \texttt{huding@ustc.edu.cn, \{yang208, hjw0330\}@mail.ustc.edu.cn}
}
\begin{document}
\maketitle

\begin{abstract}
Adversarial machine learning has attracted a great amount of attention in recent years. 
Due to the great importance of support vector machines (SVM) in machine learning, we consider defending SVM against poisoning attacks in this paper. We study two commonly used strategies for defending: designing robust SVM algorithms and data sanitization. Though several robust SVM algorithms have been proposed before, most of them either are in lack of adversarial-resilience, or rely on strong assumptions about the data distribution or the attacker's behavior. Moreover, the research on the hardness of designing a quality-guaranteed adversarially-resilient SVM algorithm is still quite limited. We are the first, to the best of our knowledge, to prove that even the simplest hard-margin one-class SVM with adversarial outliers problem is NP-complete, and has no fully PTAS unless P=NP. 
For data sanitization, we explain the effectiveness of DBSCAN (as a density-based outlier removal method) for defending against poisoning attacks. In particular, we link it to the intrinsic dimensionality by proving a sampling theorem in doubling metrics. In our empirical experiments, we systematically compare several defenses including the DBSCAN and robust SVM methods, and investigate the influences from the intrinsic dimensionality and poisoned fraction to their performances.
\end{abstract}

\vspace{-0.1in}

\section{Introduction}
\label{sec-intro}
\vspace{-0.1in}
In the past decades we have witnessed enormous progress in machine learning.  One driving force behind this is the successful applications of machine learning technologies to many different fields, such as data mining, networking, and bioinformatics. However, with its territory rapidly enlarging, machine learning has also imposed a number of new challenges. In particular,   {\em  adversarial machine learning} which concerns about the potential vulnerabilities of the algorithms, has attracted a great amount of attention~\citep{DBLP:conf/ccs/BarrenoNSJT06,huang2011adversarial,biggio2018wild,DBLP:journals/cacm/GoodfellowMP18}. 
As mentioned in the survey paper~\citep{biggio2018wild}, the very first work of adversarial machine learning dates back to 2004, in which \cite{DBLP:conf/kdd/DalviDMSV04} formulated the adversarial classification problem as a game between the classifier and the adversary. In general, the adversarial attacks against machine learning can be categorized to \textbf{evasion attacks} and \textbf{poisoning attacks}~\citep{biggio2018wild}. 
An evasion attack happens at test time, where the adversary aims to evade the trained classifier by manipulating test examples. For example, \cite{DBLP:journals/corr/SzegedyZSBEGF13} observed that small perturbation to a test image can arbitrarily change the  neural network's prediction.

%; following their work, a number of articles studied the robustness of deep learning against evasion attacks~\cite{DBLP:journals/cacm/GoodfellowMP18,DBLP:journals/corr/GoodfellowSS14,DBLP:conf/sp/PapernotM0JS16,DBLP:conf/sp/Carlini017,DBLP:conf/iclr/MadryMSTV18}. There are also several works on the robustness of other machine learning models, such as linear classifier and $k$-nearest neighbor, against evasion attacks~\cite{DBLP:conf/ccs/RussuDBFR16,DBLP:journals/corr/abs-2003-06559}. 

In this paper, we focus on poisoning attacks that happen at training time. Usually, the adversary injects a small number of  specially crafted samples into the training data which can make the decision boundary severely deviate and cause unexpected misclassification. In particular, because  open datasets are commonly used to train our machine learning algorithms nowadays, poisoning attack has become a key security issue that seriously limits real-world applications~\citep{biggio2018wild}. 
For instance, even a small number of poisoning samples can significantly increase the test error of support vector machine (SVM)~\citep{DBLP:conf/icml/BiggioNL12,DBLP:conf/aaai/MeiZ15,DBLP:conf/ecai/XiaoXE12}. Beyond linear classifiers, a number of works studied the poisoning attacks for other machine learning problems, such as clustering~\citep{DBLP:conf/sspr/BiggioBPMMPR14}, PCA~\citep{DBLP:conf/imc/RubinsteinNHJLRTT09}, and regression~\citep{DBLP:conf/sp/JagielskiOBLNL18}.

Though lots of works focused on constructing poisoning attacks, our ultimate goal is to design defenses. Poisoning samples can be regarded as {\em outliers}, and this leads to two natural approaches to defend: \textbf{(1) data sanitization defense}, {\em i.e.,} first perform outlier removal and then run an existing machine learning algorithm on the cleaned data~\citep{DBLP:conf/sp/CretuSLSK08},  or \textbf{(2) directly design a robust optimization algorithm that is resilient against outliers}~\citep{DBLP:journals/jmlr/ChristmannS04,DBLP:conf/sp/JagielskiOBLNL18}. 

%can design robust optimization algorithms ~\cite{biggio2018wild}. 

\cite{DBLP:conf/nips/SteinhardtKL17} studied two basic methods of data sanitization defense, which remove the points outside a specified sphere or slab, for binary classification; they showed that high dimensionality gives attacker more room for constructing attacks to evade outlier removal.  \cite{DBLP:journals/corr/LaishramP16} applied the seminal DBSCAN (Density-Based Spatial Clustering of Applications with Noise) method~\citep{ester1996density} to remove outliers for SVM and showed that it can successfully identify most of the poisoning data. However, their DBSCAN approach is lack of theoretical analysis. Several other outlier removal methods for fighting poisoning attacks have also been studied recently~\citep{DBLP:conf/pkdd/PaudiceML18,DBLP:journals/corr/abs-1802-03041}. Also, it is worth noting that outlier removal actually is a topic that  has been extensively studied in various fields before~\citep{chandola2009anomaly}. 

The other defense strategy,  designing robust optimization algorithms, also has a long history in the machine learning community. A substantial part of robust optimization algorithms rely on the idea of regularization. For example,  \cite{DBLP:journals/jmlr/XuCM09} studied the relation between robustness and regularization for SVM; other robust SVM algorithms  include~\citep{DBLP:journals/prl/TaxD99,conf/aaai/XuCS06,icml2014c2_suzumura14,DBLP:conf/nips/NatarajanDRT13,DBLP:journals/pr/XuCHP17,DBLP:journals/entropy/KanamoriFT17}. However, as discussed in~\citep{DBLP:conf/aaai/MeiZ15,DBLP:conf/sp/JagielskiOBLNL18}, these approaches are not quite ideal to defend against poisoning attacks \textbf{since the outliers can be located arbitrarily in the feature space  by the adversary}. Another idea for achieving the robustness guarantee is to add strong assumptions about the data distribution or the attacker's behavior~\citep{DBLP:conf/nips/FengXMY14,DBLP:journals/pr/WeerasingheEAL19}, but these assumptions are usually not well satisfied in practice.  
An alternative approach is to explicitly remove outliers during optimization, such as the ``trimmed'' method for robust regression~\citep{DBLP:conf/sp/JagielskiOBLNL18}; but this approach often results in a challenging \textbf{combinatorial optimization problem}: if $z$ of the input $n$ data items are outliers ($z<n$), we have to consider an exponentially large number ${n\choose z}$ of  different possible cases in the adversarial setting. 

%As an example, the clustering with outliers problems, like $k$-means and $k$-center with outliers, have at least quadratic time complexities in general~\cite{charikar2001algorithms,chen2008constant}; of course, if we aim to obtain only a local optimum, we can formulate the problem as a bilevel optimization that alternatively removes outliers and minimizes the objective function on the remaining data in each iteration, like the $k$-means$-$$-$ algorithm~\cite{chawla2013k} (this alternating minimization idea also works for other problems, such as regression with outliers~\cite{DBLP:journals/cacm/FischlerB81}). 
%

%Therefore, the design of efficient and robust optimization algorithms is urgently needed to meet these challenges.
%  
%  ~\cite{conf/aaai/XuCS06,icml2014c2_suzumura14,ding2015random}  
\vspace{-0.13in}
\subsection{Our Contributions}
\label{sec-contribution}
\vspace{-0.07in}

Due to the great importance  in machine learning~\citep{journals/tist/ChangL11}, we focus on defending SVM against poisoning attacks in this paper. Our contributions are twofold.

\textbf{(\rmnum{1}).} First, we consider the robust optimization approach.  To study its complexity, we only consider the hard-margin case (because the soft-margin case is more complicated and thus should have an even higher complexity). As mentioned above, we can formulate the SVM with outliers problem as a combinatorial optimization problem for achieving the \textbf{adversarial-resilience}: finding an optimal subset of $n-z$ items from the poisoned input data to achieve the largest separating margin (the formal definition is shown in Section~\ref{sec-pre}). 

Though its local optimum can be obtained by using various methods, such as the alternating minimization approach~\citep{DBLP:conf/sp/JagielskiOBLNL18}, 
%it is still unclear that how to achieve the global optimal solution. 
it is often very challenging to achieve a quality guaranteed solution for such adversarial-resilience optimization problem. For instance, \cite{DBLP:conf/icml/SimonovFGP19} showed that unless the Exponential Time Hypothesis (ETH) fails, it is impossible not only to solve the {\em PCA with outliers} problem exactly but even to approximate it within a constant factor. 
A similar hardness result was also proved for {\em linear regression with outliers}  by \cite{DBLP:journals/algorithmica/MountNPSW14}. 
But for SVM with outliers, we are unaware of any  \textbf{hardness-of-approximation result} before. We try to bridge the gap in the current state of knowledge in Section~\ref{sec-hard}. We prove that even the simplest one-class SVM with outliers problem is NP-complete, and has no fully polynomial-time approximation scheme (PTAS) unless P$=$NP. So it is quite unlikely that one can achieve a (nearly) optimal solution in polynomial time.

\textbf{(\rmnum{2}).} Second, we investigate the DBSCAN based data sanitization defense and explain its effectiveness in theory (Section~\ref{sec-alg}). DBSCAN is one of the most popular density-based clustering methods and has been implemented for solving many real-world outlier removal problems~\citep{ester1996density,schubert2017dbscan}; roughly speaking, the inliers are assumed to be located in some dense regions and the remaining points are recognized as the outliers. Actually, the intuition of using DBSCAN for data sanitization is  straightforward~\citep{DBLP:journals/corr/LaishramP16}. We assume the original input training data (before poisoning attack) is large and dense enough in the domain $\Omega$; thus the poisoning data should be the sparse outliers together with some small clusters located outside the dense regions, which can be identified by the DBSCAN. Obviously, if the attacker has a fixed budget $z$ (the number of poisoning points), the lager the data size $n$ is, the sparser the outliers appear to be (and the more efficiently the DBSCAN performs).

% (we can imagine the extreme case that $z/n$ is close to $1$, where it is clearly inappropriate to use a density based clustering method to identify the outliers). 

Thus, to guarantee the effectiveness of the DBSCAN approach, a fundamental question in theory is what about \textbf{the lower bound of the data size $n$}  (we can assume that the original input data is a set of {\em i.i.d.} samples drawn from the domain $\Omega$). 
However, to achieve a favorable lower bound is a non-trivial task. The  VC dimension~\citep{DBLP:journals/jcss/LiLS01} of the range space induced by the Euclidean distance is high in a high-dimensional feature space, and thus the lower bound of the data size $n$ can be very large.  Our idea is motivated by the recent observations on the link between the adversarial vulnerability and the intrinsic dimensionality~\citep{DBLP:journals/corr/abs-1905-01019,DBLP:conf/wifs/AmsalegBBEHNR17,DBLP:conf/iclr/Ma0WEWSSHB18}. We prove a lower bound of $n$ that depends on the intrinsic dimension of $\Omega$ and is independent of the feature space's dimensionality. 

Our result strengthens the observation from~\cite{DBLP:conf/nips/SteinhardtKL17} who only considered the Euclidean space's dimensionality: more precisely, it is the ``high intrinsic dimensionality'' that gives attacker more room  to evade outlier removal. In particular, different from the previous results on evasion attacks~\citep{DBLP:journals/corr/abs-1905-01019,DBLP:conf/wifs/AmsalegBBEHNR17,DBLP:conf/iclr/Ma0WEWSSHB18}, our result is the first one linking poisoning attacks to intrinsic dimensionality, to the best of our knowledge.  
In Section~\ref{sec-exp}, we investigate several popular defending methods (including DBSCAN), where  the intrinsic dimension of data demonstrates significant influence on their defending performances.
\vspace{-0.1in}

\section{Preliminaries}
\label{sec-pre}
\vspace{-0.07in}
%\subsection{SVM with Outliers}  

Given two point sets $P^+$ and $P^-$ in $\mathbb{R}^d$, the problem of linear \textbf{support vector machine (SVM)}~\citep{journals/tist/ChangL11} is to find the maximum margin  separating these two point sets (if they are separable). 
%For convenience, we say that $P^+$ and $P^-$ are the sets of points labeled as ``$+1$'' and ``$-1$'', respectively. 
If $P^+$ (or $P^-$) is  a single point, say the origin, the problem is called \textbf{one-class SVM}. The SVM can be formulated as a quadratic programming problem, and a number of efficient techniques have been developed in the past, such as the soft margin SVM~\citep{mach:Cortes+Vapnik:1995}, $\nu$-SVM~\citep{bb57389,conf/nips/CrispB99}, and Core-SVM~\citep{tkc-cvmfstv-05}. If $P^+$ and $P^-$ are not separable, we can apply the  kernel method: each point $p\in P^+\cup P^-$ is mapped to be $\phi (p)$ in a higher dimensional space; the inner product $\langle\phi(p_1), \phi(p_2)\rangle$ is defined by a kernel function $\mathcal{K}(p_1, p_2)$. Many existing SVM algorithms can be adapted to handle the non-separable case by using kernel functions.

\textbf{Poisoning attacks.} The adversary usually injects some bad points to the original data set $P^+\cup P^-$. For instance, the adversary can take a sample $q$ from the domain of $P^+$, and flip its label to be ``$-$''; therefore, this poisoning sample $q$ can be viewed as an outlier of $P^-$. 
%
%
%the poisoning attacks can be generated through three ways. \textbf{(1)} The adversary injects some bad points to the original data set $P^+\cup P^-$; \textbf{(2)} the adversary changes the locations of some points from $P^+\cup P^-$; \textbf{(3)} the adversary flips the labels of some points from $P^+\cup P^-$. 
%
Since poisoning attack is expensive, we often assume that the adversary can poison at most $z\in\mathbb{Z}^+$ points (or the poisoned fraction $\frac{z}{ |P^+\cup P^- |}$ is a fixed small number in $(0,1)$). 
We can formulate the  defense against poisoning attacks as the following combinatorial optimization problem. As mentioned in Section~\ref{sec-contribution}, it is sufficient to consider only the simpler hard-margin case for studying  its hardness. 

\begin{definition} [\textbf{SVM  with Outliers}]
	\label{def-svm}
	Let $(P^+, P^-)$ be an instance of SVM in $\mathbb{R}^d$, and suppose $\big|P^+\cup P^-\big|=n$. Given a positive integer $z<n$, the problem of SVM with outliers is to find two subsets $P^+_1\subseteq P^+$ and $P^-_1\subseteq P^-$ with $\big|P^+_1\cup P^-_1\big|=n-z$, such that the width of the margin separating  $P^+_1$ and $P^-_1$ is maximized.   
	% 
	% the smallest shape $c\in \mathcal{X}$ that covers $(1-\gamma)n$ points. Namely, the task is to find a subset of $P$ with size $(1-\gamma)n$ such that its minimum enclosing X shape is the smallest among all possible choices of the subset. The obtained solution is denoted by $MEX(P, \gamma)$.
	
	Suppose the optimal margin has the width $h_{opt}>0$. If we achieve a solution with the margin width $h\geq (1-\epsilon)h_{opt}$ where $\epsilon$ is a small number in $(0,1)$, we say that it is a $(1-\epsilon)$-approximation. 
\end{definition}
\begin{remark}
	The model proposed in Definition~\ref{def-svm}  follows the popular \textbf{data trimming} idea from robust statistics~\citep{books/wi/RousseeuwL87}. As an example similar with Definition~\ref{def-svm}, \cite{DBLP:conf/sp/JagielskiOBLNL18} proposed a data trimming based regression model to defend  against poisoning attacks. 
\end{remark}

%\subsection{Doubling Dimension}

We also need to clarify the  intrinsic dimensionality for our following analysis. 
\textbf{Doubling dimension}  is a measure of intrinsic dimensionality that has been widely adopted in the learning theory community~\citep{DBLP:journals/jcss/BshoutyLL09}.  
Given a point $p$ and $r\geq 0$, we use $\mathbb{B}(p, r)$ to indicate the ball of radius $r$ around $p$ in the space. 

%Usually, the doubling dimension is defined for an abstract metric space $(X, g)$ where $g$ is the distance function of $X$~\cite{DBLP:journals/talg/ChanGMZ16}. In this paper, since we mainly focus on the  high-dimensional data with low intrinsic dimensions, we directly describe the doubling dimension of point set in the context of Euclidean space. 
\vspace{-0.05in}

\begin{definition}[\textbf{Doubling Dimension}]
	\label{def-dd}
	The doubling dimension of a point set $P$ from some metric space\footnote{The space can be a Euclidean space or an abstract metric space.} is the smallest number $\rho$, such that for any $p\in P$ and $r\geq 0$, the set $P\cap \mathbb{B}(p, 2r)$ can always be  covered by the union of at most $2^\rho$ balls with radius $r$ in the space.
\end{definition}
\vspace{-0.05in}
%\begin{remark}
	To understand doubling dimension, we consider the following simple case. If the points of $P$ distribute uniformly in a $d'$-dimensional flat in $\mathbb{R}^d$,   then it is easy to see that  $P$ has the doubling dimension $\rho=O(d')$, which is independent of the Euclidean dimension $d$ ({\em e.g.,} $d$ can be much higher than $\rho$). 
	Intuitively,  doubling dimension is  used for describing the expansion rate of a given point set in the space.  
	It is worth noting that the intrinsic dimensionality described in~\citep{DBLP:conf/wifs/AmsalegBBEHNR17,DBLP:conf/iclr/Ma0WEWSSHB18} is quite similar to  doubling dimension, which also measures expansion rate. 
%\end{remark}

\vspace{-0.05in}
\section{The Hardness of SVM with Outliers}
\label{sec-hard}
%\vspace{-0.05in}
\begin{figure*}[]
	\centering
	\subfloat[]{\includegraphics[height=0.8in]{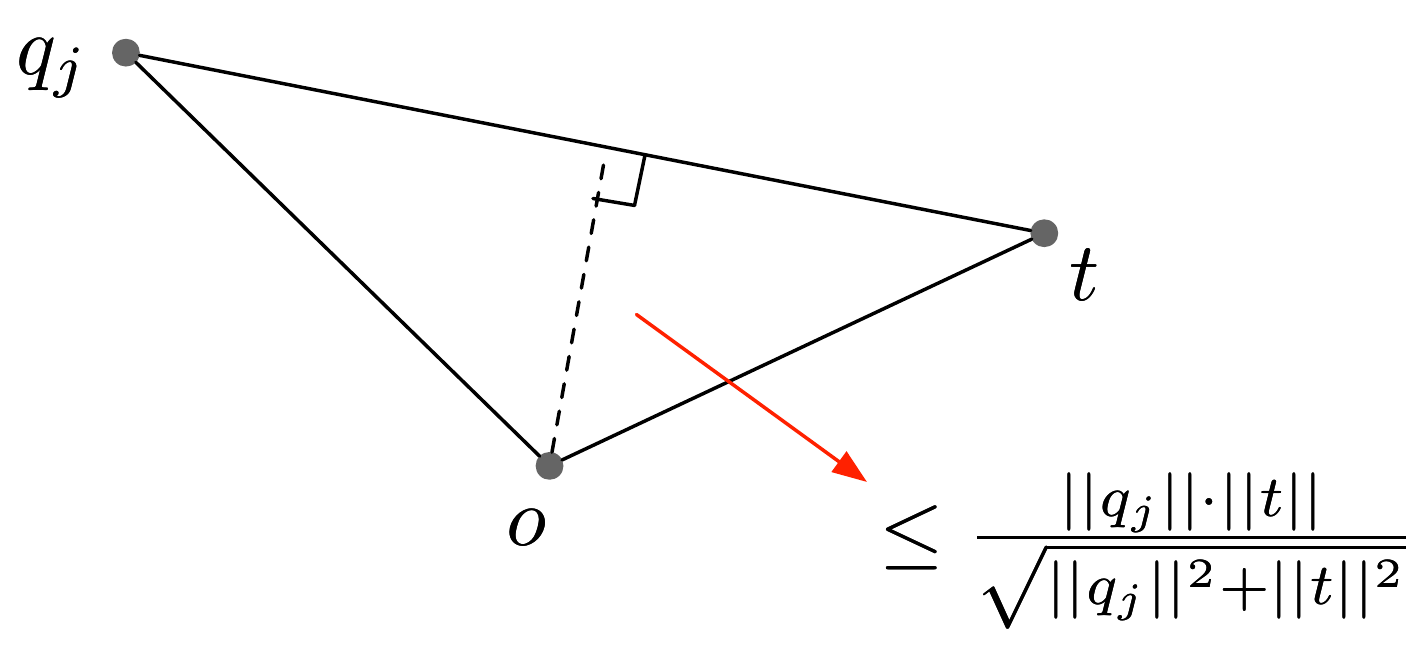}\label{fig-hard}}
	\hspace{1.2in}	
	\subfloat[]{\includegraphics[height=0.8in]{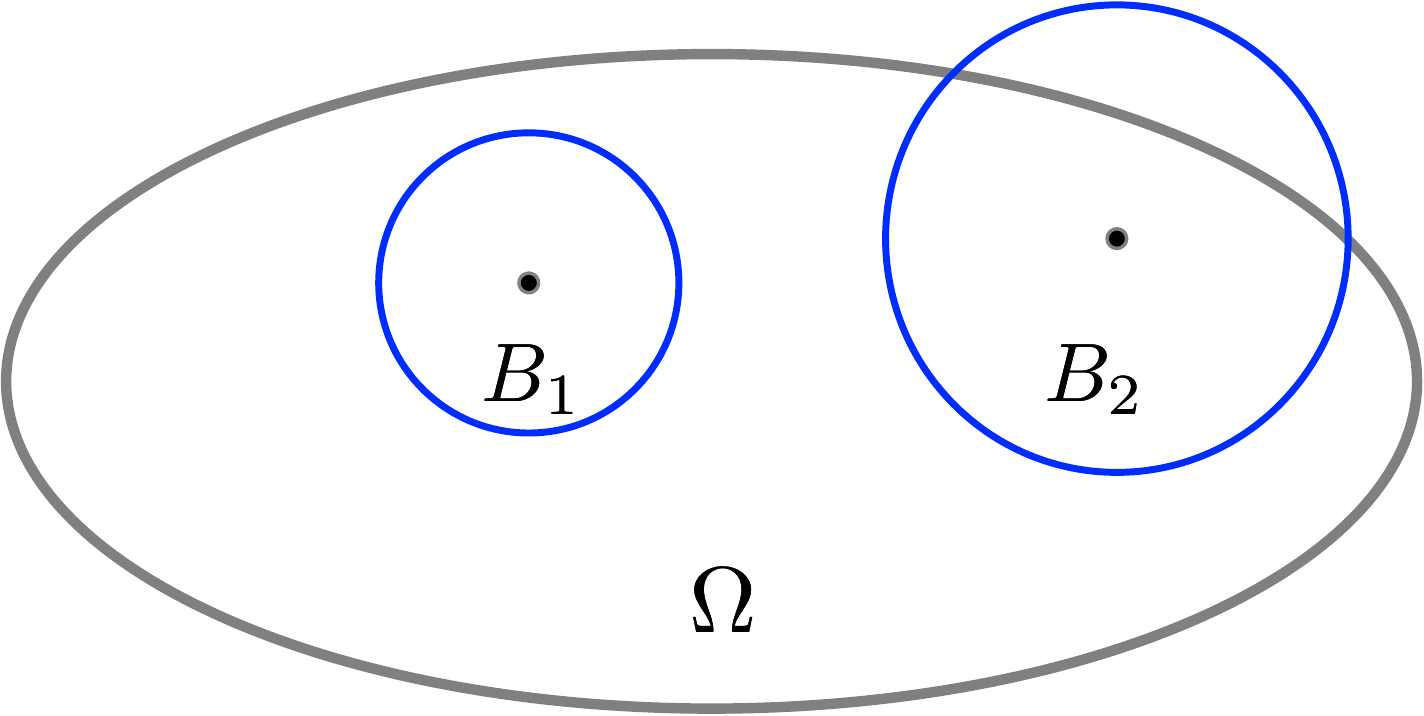}\label{fig-enclose}}
	%	\subfloat[]{\includegraphics[height=1.3in]{synthetic_eud} \label{fig-23}}
	%	\subfloat[]{\includegraphics[height=1.0in]{fscore_poi} \label{fig-24}}
	\vspace{-0.1in}
	\caption{(a) An illustration for the formula (\ref{for-hard-1}); (b) the ball $B_1$ is enclosed by $\Omega$ and the ball $B_2$ is not.}
	\label{fig-exp2}
	\vspace{-0.2in}
\end{figure*}

In this section, we prove that even the one-class SVM with outliers problem is NP-complete and has no fully PTAS unless P$=$NP (that is, we cannot achieve a polynomial time $(1-\epsilon)$-approximation for any given $\epsilon\in(0,1)$).  Our idea is partly inspired by the result from \cite{DBLP:journals/jsc/Megiddo90}. Given a set of points in $\mathbb{R}^d$, the ``covering by two balls'' problem is to determine that whether the point set can be covered by two unit balls. By the reduction from $3$-SAT, Megiddo proved that the ``covering by two balls'' problem is NP-complete. In the proof of the following theorem, we modify Megiddo's construction of the reduction to adapt the one-class SVM with outliers problem.
\vspace{-0.05in}

\begin{theorem}
	\label{the-hard}
	The one-class SVM with outliers problem is NP-complete, and has no fully PTAS unless P$=$NP. 
\end{theorem}
\vspace{-0.1in}

	Let $\Gamma$ be a $3$-SAT instance with the literal set $\{u_1, \bar{u}_1,\cdots, u_l, \bar{u}_l\}$ and clause set $\{E_1, \cdots, E_m\}$. We construct the corresponding instance $P_\Gamma$ of one-class SVM with outliers. First, let $U=\{\pm e_i\mid i=1, 2, \cdots, l+1\}$ be the $2(l+1)$ unit vectors of $\mathbb{R}^{l+1}$, where each $e_i$ has ``$1$'' in the $i$-th position and ``$0$'' in other positions. Also, for each clause $E_j$ with $1\leq j\leq m$, we generate a point $q_j=(q_{j,1}, q_{j,2}, \cdots, q_{j, l+1})$ as follows. For $1\leq i\leq l$, 
	
	\[
	q_{j, i}=\left\{
	\begin{array}{ll}
		\alpha,&\text{ if  $u_i$ occurs in $E_j$; }\\
		-\alpha, &\text{ else if $\bar{u}_i$ occurs in $E_j$;}\\
		0, &\text{ otherwise.}
	\end{array}
	\right.
	\]	
	
	% 	(\rmnum{1}) if  $u_i$ occurs in $E_j$, $q_{j, i}=\alpha$, (\rmnum{2}) else if $\bar{u}_i$ occurs in $E_j$, $q_{j, i}=-\alpha$, (\rmnum{3}) else,  $q_{j, i}=0$; 
	In addition, $q_{j, l+1}=3\alpha$. For example, if $E_j=u_{i_1}\vee \bar{u}_{i_2}\vee u_{i_3}$, the point 
	\begin{eqnarray}
		q_j=(0, \cdots, 0, \underset{i_1}{\alpha}, 0, \cdots, 0, \underset{i_2}{-\alpha},0, \cdots,\nonumber\\
		0, \underset{i_3}{\alpha}, 0, \cdots, 0, 3\alpha).
	\end{eqnarray}
	The value of $\alpha$ will be determined later. Let $Q$ denote the set $\{q_1, \cdots, q_m\}$. Now, we construct the instance $P_\Gamma=U\cup Q$ of one-class SVM with outliers, where the number of points $n=2(l+1)+m$ and the number of outliers $z=l+1$.  
	%	Below we prove that $\Gamma$ has a satisfying assignment if and only if $P_\Gamma$ has a solution with margin width $\frac{1}{\sqrt{l+1}}$. 
	Then we have the following lemma.
	
	\begin{lemma}
		\label{lem-hard}
		Let $\alpha>1/2$. $\Gamma$ has a satisfying assignment if and only if $P_\Gamma$ has a solution with margin width $\frac{1}{\sqrt{l+1}}$.
	\end{lemma}
	\begin{proof}
		\textbf{First, we suppose there exists a satisfying assignment for $\Gamma$.} We define the set $S\subset P_{\Gamma}$ as follows. If $u_i$ is true in $\Gamma$, we include $e_i$ in $S$, else, we include $-e_i$ in $S$; we also include $e_{l+1}$ in $S$.  We claim that the set $S\cup Q$ yields a solution of the instance $P_\Gamma$ with the margin width $\frac{1}{\sqrt{l+1}}$, that is, the size $|S\cup Q|=n-z$ and the margin separating the origin $o$ and $S\cup Q$ has width $\frac{1}{\sqrt{l+1}}$. It is easy to verify the size of $S\cup Q$. To compute the width, we consider the mean point of $S$ which is denoted as $t$. For each $1\leq i\leq l$, if $u_i$ is true, the $i$-th position of $t$ should be $\frac{1}{l+1}$, else, the $i$-th position of $t$ should be $-\frac{1}{l+1}$; the $(l+1)$-th position of $t$ is $\frac{1}{l+1}$. Obviously, $||t||=\frac{1}{\sqrt{l+1}}$. Let $\mathcal{H}_t$ be the hyperplane that is orthogonal to the vector $t-o$ and passing through $t$. So $\mathcal{H}_t$ separates $S$ and $o$ with the margin width $||t||=\frac{1}{\sqrt{l+1}}$. Furthermore, for any point $q_j\in Q$, since there exists at least one true variable in $E_j$, we have the inner product 
		\begin{eqnarray}
			\langle q_j, \frac{t}{||t||}\rangle &\geq& \frac{3\alpha}{\sqrt{l+1}}+\frac{\alpha}{\sqrt{l+1}}-\frac{2\alpha}{\sqrt{l+1}}\nonumber\\
			&=&\frac{2\alpha}{\sqrt{l+1}}
			>\frac{1}{\sqrt{l+1}},
		\end{eqnarray}
		where the last inequality comes from the fact $\alpha>1/2$. Therefore, all the points from $Q$ lie on the same side of $\mathcal{H}_t$ as $S$, and then the set $S\cup Q$ can be separated from $o$ by a margin with width $\frac{1}{\sqrt{l+1}}$.

		\textbf{Second, suppose the instance $P_\Gamma$ has a solution with margin width $\frac{1}{\sqrt{l+1}}$.} With a slight abuse of notations, we still use $S$ to denote the subset of $U$ that is included in the set of $n-z$ inliers. Since the number of outliers is $z=l+1$, we know that for any pair $\pm e_i$, there exists exactly one point belonging to $S$; also, the whole set $Q$ should be included in the set of inliers so as to guarantee that there are $n-z$ inliers in total. We still use $t$ to denote the mean point of $S$ ($||t||=\frac{1}{\sqrt{l+1}}$). Now, we design the assignment for $\Gamma$: if $e_i\in S$, we assign $u_i$ to be true, else, we assign $\bar{u}_i$ to be true. We claim that $\Gamma$ is satisfied by this assignment. For any clause $E_j$, if it is not satisfied, {\em i.e.,} all the three variables in $E_j$ are false, then we have the inner product
		\begin{eqnarray}
			\langle q_j, \frac{t}{||t||}\rangle\leq  \frac{3\alpha}{\sqrt{l+1}}-\frac{3\alpha}{\sqrt{l+1}}=0.
		\end{eqnarray}
		That means the angle $\angle q_j o t\geq \pi/2$. So any margin separating the origin $o$ and the set $S\cup Q$ should has the width at most 
		\begin{eqnarray}
			\frac{||q_j||\cdot ||t||}{\sqrt{||q_j||^2+||t||^2}}<||t||=\frac{1}{\sqrt{l+1}}.\label{for-hard-1}
		\end{eqnarray}
		See Figure~\ref{fig-hard} for an illustration. This is in contradiction to the assumption that $P_\Gamma$ has a solution with margin width $\frac{1}{\sqrt{l+1}}$.

		Overall, $\Gamma$ has a satisfying assignment if and only if $P_\Gamma$ has a solution with margin width $\frac{1}{\sqrt{l+1}}$. 	
	\end{proof}
		\vspace{-0.05in}
	Now we are ready to prove the theorem. 
	\vspace{-0.05in}
	\begin{proof}(\textbf{of Theorem~\ref{the-hard}})
	Since 3-SAT is NP-complete, Lemma~\ref{lem-hard} implies that the one-class SVM with outliers problem is NP-complete too; otherwise, we can determine that whether a given instance $\Gamma$ is satisfiable by computing the optimal solution of $P_\Gamma$. Moreover, the gap between $\frac{1}{\sqrt{l+1}}$ and $\frac{||q_j||\cdot ||t||}{\sqrt{||q_j||^2+||t||^2}}$ (from the formula (\ref{for-hard-1})) is 
	\begin{eqnarray}
		&\frac{1}{\sqrt{l+1}}- \sqrt{\frac{12\alpha^2\frac{1}{l+1}}{12\alpha^2+\frac{1}{l+1}}}\nonumber\\
		=&(\frac{1}{l+1})^{3/2}\frac{1}{\sqrt{12\alpha^2+\frac{1}{l+1}}(\sqrt{12\alpha^2+\frac{1}{l+1}}+2\sqrt{3}\alpha)}\nonumber\\
		=&\Theta\big((\frac{1}{l+1})^{3/2}\big),
	\end{eqnarray}
	if we assume $\alpha$ is a fixed constant. Therefore, if we set $\epsilon=O\big(\frac{(\frac{1}{l+1})^{3/2}}{(\frac{1}{l+1})^{1/2}}\big)=O(\frac{1}{l+1})$, then $\Gamma$ is satisfiable if and only if any $(1-\epsilon)$-approximation of the instance $P_\Gamma$ has width $>\sqrt{\frac{12\alpha^2\frac{1}{l+1}}{12\alpha^2+\frac{1}{l+1}}}$. That means if we have a fully PTAS for the one-class SVM with outliers problem, we can determine that whether $\Gamma$ is satisfiable or not in polynomial time. 
	%any $(1-\epsilon)$-approximation of the instance $P_\Gamma$ should be an optimal solution.
	In other words, we cannot even achieve a fully PTAS for one-class SVM with outliers, unless P$=$NP. 
\end{proof}

\section{The Data Sanitization Defense}  
\label{sec-alg}

From Theorem~\ref{the-hard}, we know that it is extremely challenging to achieve the optimal solution even for one-class SVM with outliers. Therefore, we turn to consider the other approach, data sanitization defense, under some reasonable assumption in practice.
% We assume that the original input point sets $P^+$ and $P^-$ are {\em i.i.d.} samples drawn from the domains with low doubling dimensions. For example, we can imagine that they lie on some low-dimensional manifolds in the space. 
First, we prove a general sampling theorem in Section~\ref{sec-the}. 
%which can help us to analyze density-based clustering methods on data with low doubling dimensions. 
Then, we apply this theorem to explain the effectiveness of DBSCAN for defending against poisoning attacks in Section~\ref{sec-our}. 

\subsection{A Sampling Theorem}
\label{sec-the}

Let $P$ be a set of {\em i.i.d.} samples drawn from a connected and compact domain $\Omega$ who has the doubling dimension $\rho>0$. For ease of presentation, we assume that $\Omega$ lies on a manifold $\mathcal{F}$ in the space. Let $\Delta$ denote the diameter of $\Omega$, {\em i.e.,} $\Delta=\sup_{p_1, p_2\in\Omega}||p_1-p_2||$. Also, we let $f$ be the probability density function of the data distribution over $\Omega$. 

To measure the uniformity of $f$, we define a value $\lambda$ as follows. For any  $c\in \Omega$ and any $r>0$, we say ``the ball $\mathbb{B}(c, r)$ is enclosed by $\Omega$'' if $\partial \mathbb{B}(c, r)\cap \mathcal{F}\subset\Omega$; intuitively, if the ball center $c$ is close to the boundary $\partial \Omega$ of $\Omega$ or the radius $r$ is too large, the ball will not be enclosed by $\Omega$. See Figure~\ref{fig-enclose} for an illustration. We define $\lambda\coloneqq\sup_{c, c', r}\frac{\int_{\mathbb{B}(c', r)}\! f(x) \, \mathrm{d}x}{\int_{\mathbb{B}(c, r)}\! f(x) \, \mathrm{d}x}$, where $\mathbb{B}(c, r)$ and $\mathbb{B}(c', r)$ are any two equal-sized balls, and $\mathbb{B}(c, r)$ is required to be enclosed by $\Omega$. As a simple example, if $\Omega$ lies on a flat manifold and the data  uniformly distribute over $\Omega$, the value $\lambda$ will be equal to $1$. On the other hand, if the distribution is very imbalanced or the manifold $\mathcal{F}$ is very rugged, the value $\lambda$ can be high. 

\begin{theorem}
	\label{the-sample}
	Let $m\in\mathbb{Z}^+$, $\epsilon\in (0,\frac{1}{8})$, and $\delta\in (0, \Delta)$. If the sample size 
	\begin{eqnarray} 
		|P|> \max\Big\{\Theta\big(\frac{m}{1-\epsilon}\cdot\lambda\cdot(\frac{1+\epsilon}{1-\epsilon}\frac{\Delta}{\delta})^\rho\big),\nonumber\\ \tilde{\Theta}\big(\rho\cdot\lambda^2\cdot(\frac{1+\epsilon}{1-\epsilon}\frac{\Delta}{\delta})^{2\rho}(\frac{1}{\epsilon})^{\rho+2}\big)\Big\},
	\end{eqnarray}
	then with constant probability, for any ball $\mathbb{B}\big(c,  \delta\big)$ enclosed by $\Omega$, the size $|\mathbb{B}\big(c,  \delta\big)\cap P|> m$. 
	The asymptotic notation $\tilde{\Theta}(f)=\Theta\big(f\cdot \mathtt{polylog}(\frac{L\Delta}{\delta\epsilon})\big)$.
\end{theorem}
\begin{remark}
	\textbf{(\rmnum{1})} A highlight of Theorem~\ref{the-sample} is that the lower bound of $|P|$ is independent of the dimensionality of the input space (which could be much higher than the intrinsic dimension). 
	%It is particularly important when the dimension is  very high.
	%); so if the doubling dimension $\rho$ is a fixed number, the required sample size for $P$ should be low.  
	
	\textbf{(\rmnum{2})} For the simplest case that $\Omega$ lies on a flat manifold and the data  uniformly distribute over $\Omega$, $\lambda$ will be equal to $1$ and thus the lower bound of $|P|$ in Theorem~\ref{the-sample} becomes $\max\Big\{\Theta\big(\frac{m}{1-\epsilon}(\frac{1+\epsilon}{1-\epsilon}\frac{\Delta}{\delta})^\rho\big), \tilde{\Theta}\big(\rho(\frac{1+\epsilon}{1-\epsilon}\frac{\Delta}{\delta})^{2\rho}(\frac{1}{\epsilon})^{\rho+2}\big)\Big\}$.
\end{remark}
Before proving Theorem~\ref{the-sample}, we need to relate  the doubling dimension $\rho$ to the VC dimension $\mathtt{dim}$ of the range space consisting of all balls with different radii~\citep{DBLP:journals/jcss/LiLS01}. Unfortunately, 
\cite{DBLP:conf/focs/HuangJLW18} recently showed that ``{\em although both dimensions are subjects of extensive research, to the best of our knowledge, there is no nontrivial relation known between the two}''. For instance, they constructed a doubling metric having unbounded VC dimension, and the other direction cannot be bounded neither. However, if allowing a small distortion to the distance, we can achieve an upper bound on the VC dimension for a given metric space with bounded doubling dimension. 
For stating the result, they defined a distance function called ``{\em $\epsilon$-smoothed distance function}'': $g(p, q)\in (1\pm\epsilon)||p-q||$ for any two data points $p$ and $q$, where $\epsilon\in(0,\frac{1}{8})$. Given a point $p$ and $\delta>0$, the ball defined by this distance function $g(\cdot, \cdot)$ is denoted by $\mathbb{B}_g(p,\delta)=\{q\in \text{the input space}\mid g(p, q)\leq \delta\}$.

\begin{theorem}[\cite{DBLP:conf/focs/HuangJLW18}]
	\label{the-jian}
	Suppose the point set $P$ has the doubling dimension $\rho>0$. There exists an $\epsilon$-smoothed distance function ``$g(\cdot, \cdot)$'' such that the VC dimension\footnote{\cite{DBLP:conf/focs/HuangJLW18} used ``shattering dimension'' to state their result. Actually, the shattering dimension is another measure for the complexity of range space, which is tightly related to the VC dimension~\citep{DBLP:conf/stoc/FeldmanL11}. For example, if the shattering dimension is $\rho_0$, the VC dimension should be bounded by $O(\rho_0\log \rho_0)$.} $\mathtt{dim}_\epsilon$ of the range space consisting of all balls with different radii is at most $\tilde{O}(\frac{\rho}{\epsilon^\rho})$, if replacing the Euclidean distance by $g(\cdot, \cdot)$.
\end{theorem}

\begin{proof}\textbf{(of Theorem~\ref{the-sample})} 
	%Let $c$ be any point in $ \Omega$.  First, because 
	%the density function $f$ has  Lipschitz constant $L>0$, we have 
	%%\begin{eqnarray}
	%$\frac{\int_{\mathbb{B}(c, r)}\! f(x) \, \mathrm{d}x}{\int_{\mathbb{B}(c', r)}\! f(x) \, \mathrm{d}x}\geq \frac{\alpha}{\alpha+L\cdot \Delta}=\frac{1}{\lambda}$ 
	%%\end{eqnarray}
	%for any $c'\in\Omega$ and $r>0$. Further, 
	%
	Let $r$ be any positive number. First, since the doubling dimension of $\Omega$ is $\rho$, if recursively applying Definition~\ref{def-dd} $\log\frac{\Delta}{r}$ times, we know that $\Omega$ can be covered by at most $\Theta\Big(\big(\frac{\Delta}{r}\big)^\rho\Big)$ balls with radius $r$. Thus, if $\mathbb{B}(c, r)$ is enclosed by $\Omega$, we have 
	\begin{eqnarray}
		\frac{\int_{\mathbb{B}(c, r)}\! f(x) \, \mathrm{d}x}{\int_{\Omega}\! f(x) \, \mathrm{d}x}\geq \Theta\Big(\frac{1}{\lambda}\cdot(\frac{r}{\Delta})^\rho\Big). \label{for-the-sample-1}
	\end{eqnarray}
	%Consequently, if $|P|>\Theta\big(\frac{z}{1-\epsilon}(\frac{L\Delta}{r})^\rho\big)$, we have 
	%\begin{eqnarray}
	%|\mathbb{B}(c, r)\cap P|> \Theta\Big((\frac{r}{L\Delta})^\rho\Big)\cdot |P|=\frac{z}{1-\epsilon}. \label{for-the-sample-2}
	%\end{eqnarray}
	Now we consider the size $|\mathbb{B}(c, \delta)\cap P|$. From Theorem~\ref{the-jian}, we know that the VC dimension $\mathtt{dim}_\epsilon$ with respect to the $\epsilon$-smoothed distance is $\tilde{O}(\frac{\rho}{\epsilon^\rho})$. Thus, for any $\epsilon_0\in(0,1)$, if 
	\begin{eqnarray}
		|P|\geq \Theta\big(\frac{1}{\epsilon^2_0}\mathtt{dim}_\epsilon\log\frac{\mathtt{dim}_\epsilon}{\epsilon_0}\big),  \label{for-the-sample-6}
	\end{eqnarray}
	the set $P$ will be an $\epsilon_0$-sample of $\Omega$; that is, for any point $c$ and $\delta'\geq 0$, 
	\begin{eqnarray}
		\frac{|\mathbb{B}_g(c, \delta')\cap P|}{|P|}\in \frac{\int_{\mathbb{B}_g(c, \delta')}\! f(x) \, \mathrm{d}x}{\int_{\Omega}\! f(x) \, \mathrm{d}x}\pm\epsilon_0 \label{for-the-sample-5}
	\end{eqnarray}
	with constant probability\footnote{The exact probability comes from the success probability that $P$ is an $\epsilon_0$-sample of $\Omega$.  Let $\eta\in(0,1)$, and the size $|P|$ in (\ref{for-the-sample-6}) should be at least $\Theta\big(\frac{1}{\epsilon^2_0}(\mathtt{dim}_\epsilon\log\frac{\mathtt{dim}_\epsilon}{\epsilon_0}+\log\frac{1}{\eta})\big)$ to guarantee a success probability $1-\eta$. For  convenience, we assume $\eta$ is a fixed small constant and simply say ``$1-\eta$'' is a ``constant probability''.}~\citep{DBLP:journals/jcss/LiLS01}. Because $g(\cdot, \cdot)$ is an $\epsilon$-smoothed distance function of the Euclidean distance, we have
	\begin{eqnarray}
		\mathbb{B}(c, \frac{\delta'}{1+\epsilon})\subseteq\mathbb{B}_g(c, \delta')\subseteq\mathbb{B}(c, \frac{\delta'}{1-\epsilon}).\label{for-the-sample-7-1}
	\end{eqnarray} 
	So if we set $\epsilon_0=\epsilon\cdot\Theta\Big(\frac{1}{\lambda}\cdot(\frac{1-\epsilon}{1+\epsilon}\frac{\delta}{\Delta})^\rho\Big)$ and $\delta'=(1-\epsilon)\delta$, (\ref{for-the-sample-1}), (\ref{for-the-sample-5}), and (\ref{for-the-sample-7-1}) jointly imply $\frac{|\mathbb{B}(c, \delta)\cap P|}{|P|}=$ 
	\begin{eqnarray}
		\frac{|\mathbb{B}(c, \frac{\delta'}{1-\epsilon})\cap P|}{|P|}&\geq& \frac{|\mathbb{B}_g(c, \delta')\cap P|}{|P|}\nonumber\\
		&\geq&\frac{\int_{\mathbb{B}_g(c, \delta')}\! f(x) \, \mathrm{d}x}{\int_{\Omega}\! f(x) \, \mathrm{d}x}-\epsilon_0\nonumber\\
		&\geq& \frac{\int_{\mathbb{B}(c,  \frac{\delta'}{1+\epsilon})}\! f(x) \, \mathrm{d}x}{\int_{\Omega}\! f(x) \, \mathrm{d}x}-\epsilon_0\nonumber\\
		\geq&&\hspace{-0.3in} (1- \epsilon)\cdot \Theta\Big(\frac{1}{\lambda}\cdot(\frac{1-\epsilon}{1+\epsilon}\frac{\delta}{\Delta})^\rho\Big). \label{for-the-sample-4}
	\end{eqnarray}
	The last inequality comes from (\ref{for-the-sample-1}) (since we assume the ball $\mathbb{B}\big(c,  \delta\big)$ is enclosed by $\Omega$, the shrunk ball $\mathbb{B}(c,  \frac{\delta'}{1+\epsilon})=\mathbb{B}(c,  \frac{1-\epsilon}{1+\epsilon}\delta)$ should be enclosed as well). 
	Moreover, if 
	\begin{eqnarray}
		|P|\geq \Theta\big(\frac{m}{1-\epsilon}\cdot\lambda\cdot(\frac{1+\epsilon}{1-\epsilon}\frac{\Delta}{\delta})^\rho\big),\label{for-the-sample-3}
	\end{eqnarray}
	we have $|\mathbb{B}\big(c,  \delta\big)\cap P|> m$ from (\ref{for-the-sample-4}). Combining (\ref{for-the-sample-6}) and (\ref{for-the-sample-3}), we obtain the lower bound of $|P|$.
\end{proof}
\vspace{-0.2in}
\subsection{The  DBSCAN Approach}
\label{sec-our}
\vspace{-0.1in}

%Given a point $q$ and a point set $P$ in $\mathbb{R}^d$, we use $dist(q, P)$ to denote the value $\min_{p\in P}||q-p||$ that is the shortest distance from $q$ to $P$.
For the sake of completeness, we briefly introduce the method of DBSCAN~\citep{ester1996density}. 
Given two parameters $r>0$ and $\mathtt{MinPts}\in\mathbb{Z}^+$, the DBSCAN divides the set $P$ into three classes: 
(1) $p$ is a \textbf{core point}, if $|\mathbb{B}(p, r)\cap P|> \mathtt{MinPts}$;
(2) $p$ is a \textbf{border point}, if $p$ is not a core point but $p\in \mathbb{B}(q, r)$ of some core point $q$;
(3) all the other points are \textbf{outliers}. 
%For convenience, we use  $P_{out}$ to denote the sets of outliers.
Actually, we can imagine that the set $P$ forms a graph where any pair of core or border points are connected if their pairwise distance is no larger than $r$; then the set of core points and border points form several clusters where each cluster is a connected component (a border point may belong to multiple clusters, but we can arbitrarily assign it to only one cluster). The goal of  DBSCAN is to identify  these clusters and the remaining outliers. Several efficient implementations for DBSCAN can be found in~\citep{gan2015dbscan,schubert2017dbscan}.

Following Section~\ref{sec-the}, we assume that $P$ is a set of {\em i.i.d.} samples drawn from the connected and compact domain $\Omega$  who has the doubling dimension $\rho>0$. 
We let $Q$ be the set of $z$ poisoning data items injected by the attacker to $P$, and suppose each $q\in Q$ has distance larger than $\delta_1>0$ to $\Omega$.  
In an evasion attack, we often use the adversarial perturbation distance to evaluate the attacker's capability; but in a poisoning attack, the attacker can easily achieve a large perturbation distance ({\em e.g.}, in the SVM problem, if the attacker flips the label of some point $p$, it will become an outlier having the perturbation distance larger than $h_{opt}$ to its ground truth domain, where $h_{opt}$ is the optimal margin width). Also, we assume the boundary $\partial \Omega$ is smooth and has curvature radius at least $\delta_2>0$ everywhere. For simplicity, let $\delta=\min\{\delta_1, \delta_2\}$. 
The following theorem states the effectiveness of the DBSCAN with respect to the poisoned dataset $P\cup Q$. We assume the poisoned fraction $\frac{|Q|}{|P|}=\frac{z}{|P|}<1$.

%\vspace{-0.01in}
\begin{theorem}
	\label{the-dbscan}
	We let $m$ be any absolute constant number larger than $1$, and assume that the size of $P$ satisfies the lower bound of Theorem~\ref{the-sample}. 
	If we set $r=\delta$ and $\mathtt{MinPts}=m$, and run DBSCAN on the poisoned dataset $P\cup Q$, the obtained largest cluster should be exactly $P$. In other word, the set $Q$ should be formed by the outliers and the clusters except the largest one from the DBSCAN.  
\end{theorem}
%\vspace{-0.03in}
%     \vspace{-0.05in}
\begin{proof}
	Since $\delta\leq \delta_2$, for any $p\in P$, either the ball $\mathbb{B}(p, \delta)$ is enclosed by $\Omega$, or $p$ is covered by some ball $\mathbb{B}(q, \delta)$ enclosed by $\Omega$. 
	We set $r=\delta$ and $\mathtt{MinPts}=m$, and hence from Theorem~\ref{the-sample} we know that all the points of $P$ will be core points or border points. Moreover, any point $q$ from $Q$ has distance larger than $r$ to the points of $P$, that is, any two points $q\in Q$ and $p\in P$ should not belong to the same cluster of the DBSCAN. Also, because the domain $\Omega$ is connected and compact, the set $P$ must form the largest cluster. 
\end{proof}
\begin{remark}
	\label{rem-dbscan}
	\textbf{(\rmnum{1})} We often adopt the poisoned fraction $\frac{z}{|P|}$ as the measure to indicate the attacker's capability.
	If we fix the value of $z$, the bound of $|P|$ from Theorem~\ref{the-sample} reveals that 
	%in order to correctly identify the poisoned data $Q$, the poisoned fraction $\frac{z}{|P|}$ should be at most 
	%\begin{eqnarray}
	%\Theta\big((1-\epsilon)(\frac{(1-\epsilon)r}{L\Delta})^\rho\big). 
	%\end{eqnarray}
	%Namely, 
	the larger the doubling dimension $\rho$, the lower the poisoned fraction $\frac{z}{|P|}$ (and the easier corrupting the DBSCAN defense). In addition, when $\delta$ is large, {\em i.e.,} each poisoning point has large perturbation distance and $\partial\Omega$ is sufficiently smooth, it will be relatively easy for  DBSCAN to defend. 
	
	But we should point out that \textbf{this theoretical bound probably is overly conservative}, since it requires a ``perfect'' sanitization result that removes all the poisoning samples (this is not always a necessary condition for achieving a good defending performance in practice). In our experiments, we show that the DBSCAN method can achieve promising performance, even when the poisoned fraction is higher than the threshold.
	
	\textbf{(\rmnum{2})} In practice, we  cannot obtain the exact values of $\delta$ and $m$. We follow the strategy that was commonly used in  DBSCAN implementations before~\citep{gan2015dbscan,schubert2017dbscan}; we set $\mathtt{MinPts}$ to be a small constant and tune the value of  $r$ until the largest cluster has $|P\cup Q|-z$ points. 
	%may only estimate a reasonable lower bound $\hat{\delta}$ for $\delta$. Thus, we can set $r=\hat{\delta}$ and tune the value of  $\mathtt{MinPts}$ until the largest cluster has $|P\cup Q|-z$ points. 
\end{remark}
%\vspace{-0.02in}
    \vspace{-0.15in}

%Directly solving such a high-dimensional  DBSCAN instance is very expensive, 
%A bottleneck of the original DBSCAN algorithm is that it needs to perform a range query for each data item, {\em i.e.,} computing the number of neighbors within the distance $r$, and 
%where the overall time complexity can be as large as quadratic in the number of input data points. To speed up the step of range query, a natural idea is using some efficient index structures, such as $R^*$-tree~\citep{DBLP:conf/sigmod/BeckmannKSS90} 
%,  though the overall complexity in the worst case is still $O(n^2 d)$ 
%(we refer the reader to the recent articles that systematically discussed this issue~\citep{gan2015dbscan,schubert2017dbscan}). 

\textbf{Putting it all together.} Let $(P^+, P^-)$ be an instance of SVM with $z$ outliers, where $z$ is the number of poisoning points. We assume that the original input point sets $P^+$ and $P^-$ (before the poisoning attack) are {\em i.i.d.} samples drawn respectively from the connected and compact domains $\Omega^+$ and $\Omega^-$ with doubling dimension $\rho$. 
%Also we assume each poisoned point has a distance larger than $\delta>0$ to its original domain $\Omega^+$ or $\Omega^-$. 
Then, we perform the DBSCAN procedure on $P^+$ and $P^-$ respectively (as Remark~\ref{rem-dbscan} (\rmnum{2})). Suppose the obtained  largest clusters are $\tilde{P}^+$ and $\tilde{P}^-$. Finally, we run an existing SVM algorithm on the cleaned instance $(\tilde{P}^+, \tilde{P}^-)$.

% References and End of Paper
% These lines must be placed at the end of your paper

\vspace{-0.15in}
\section{Empirical Experiments}
\label{sec-exp}
\vspace{-0.1in}
\begin{figure*}[]
	\centering
	\subfloat[]{\includegraphics[height=1.2in]{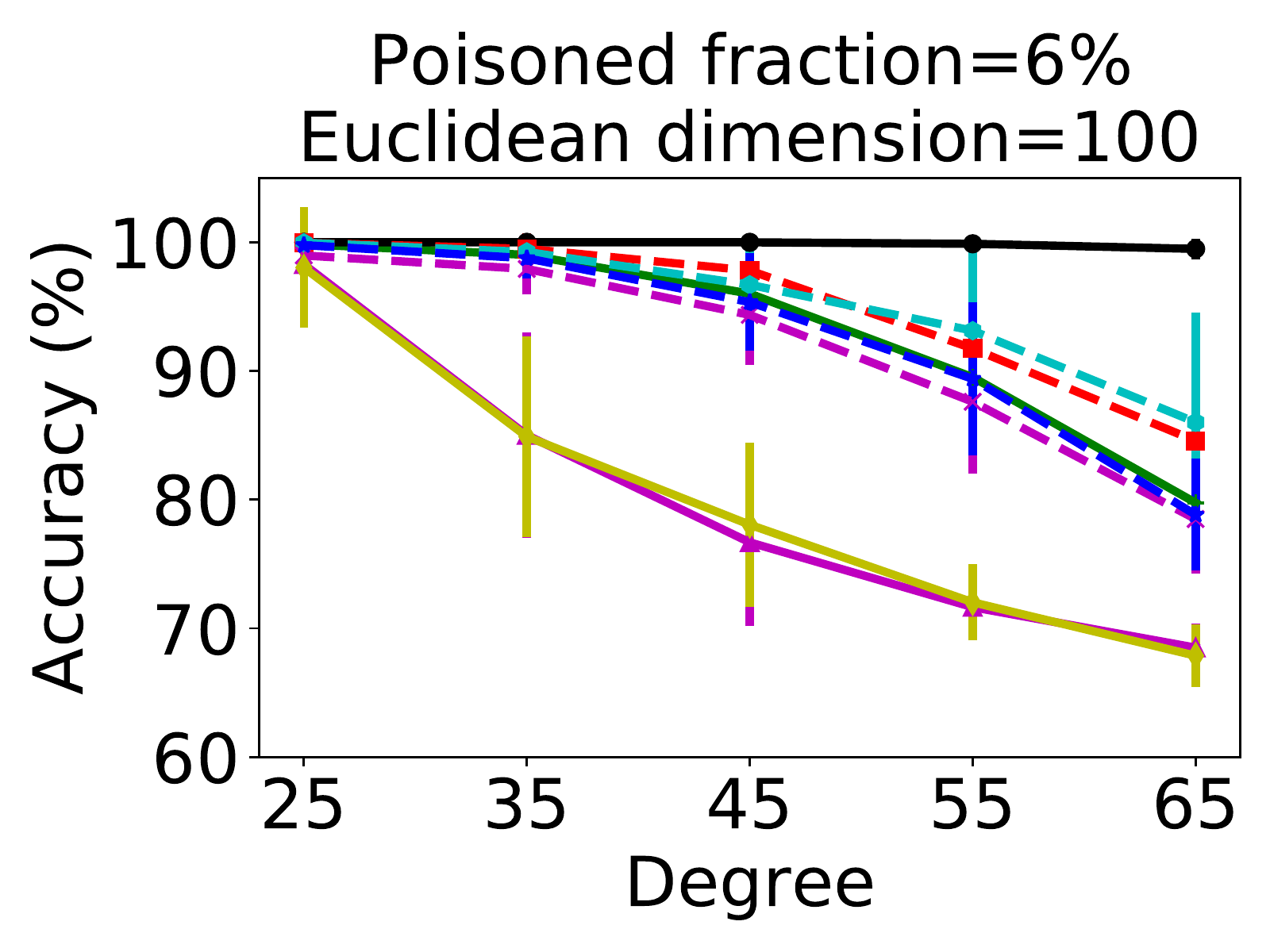}\label{fig-s10}}
	\hspace{0.5in}
	\subfloat[]{\includegraphics[height=1.2in]{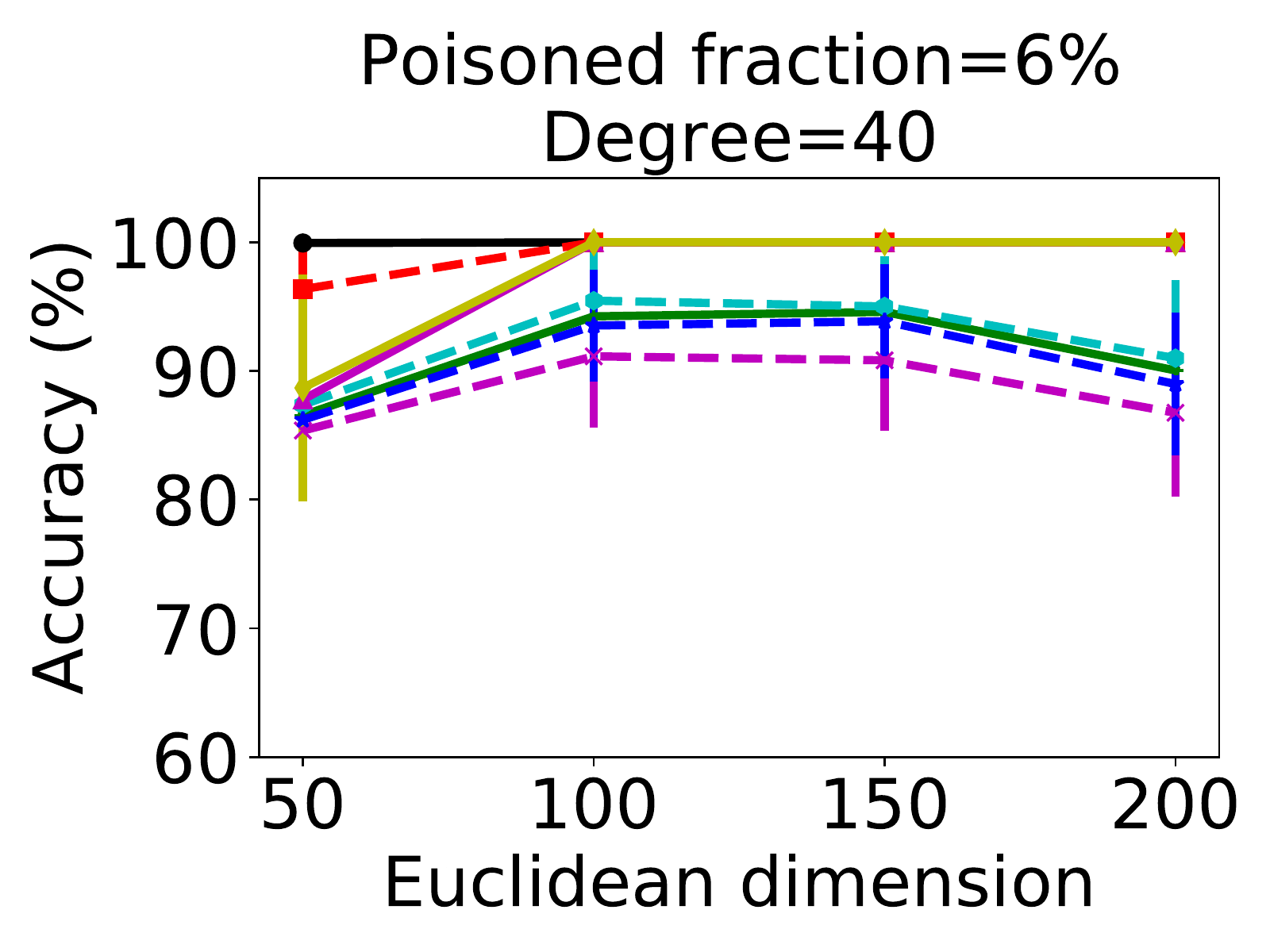}\label{fig-s11}}
	\hspace{0.5in}	
	\subfloat[]{\includegraphics[height=1.2in]{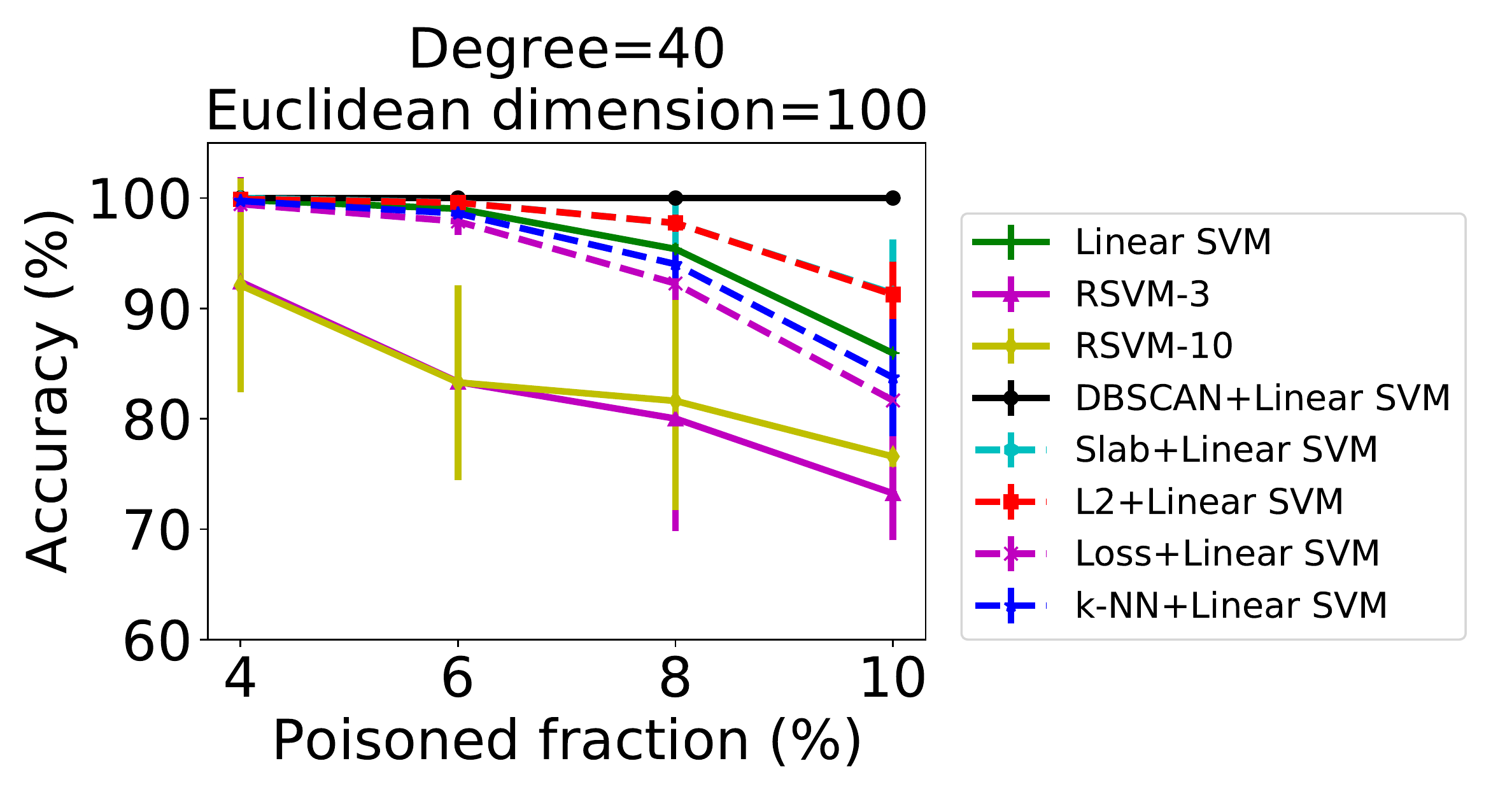}\label{fig-s12}}
	
	\subfloat[]{\includegraphics[height=1.2in]{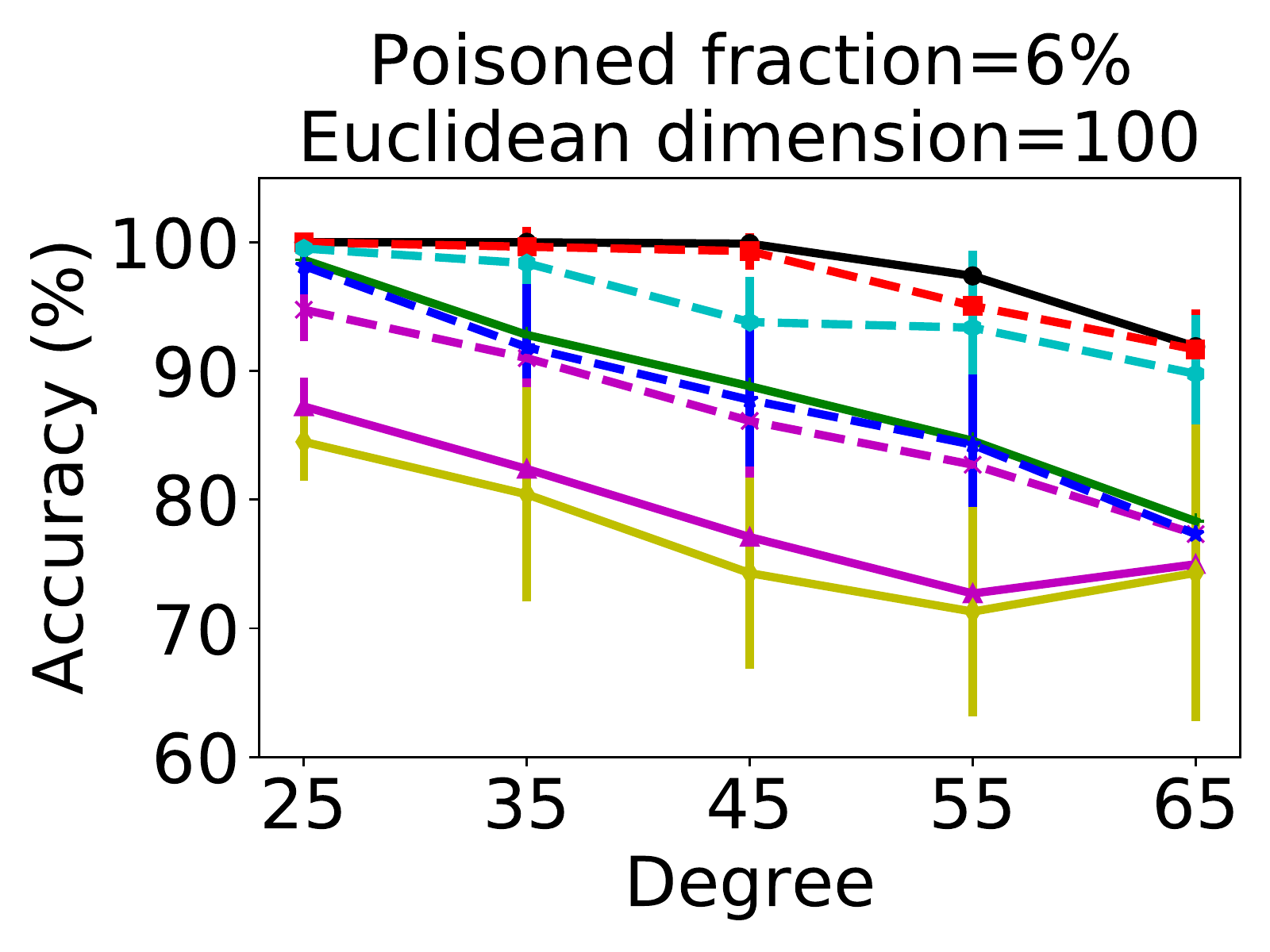}\label{fig-s20}}
	\hspace{0.5in}	
	\subfloat[]{\includegraphics[height=1.2in]{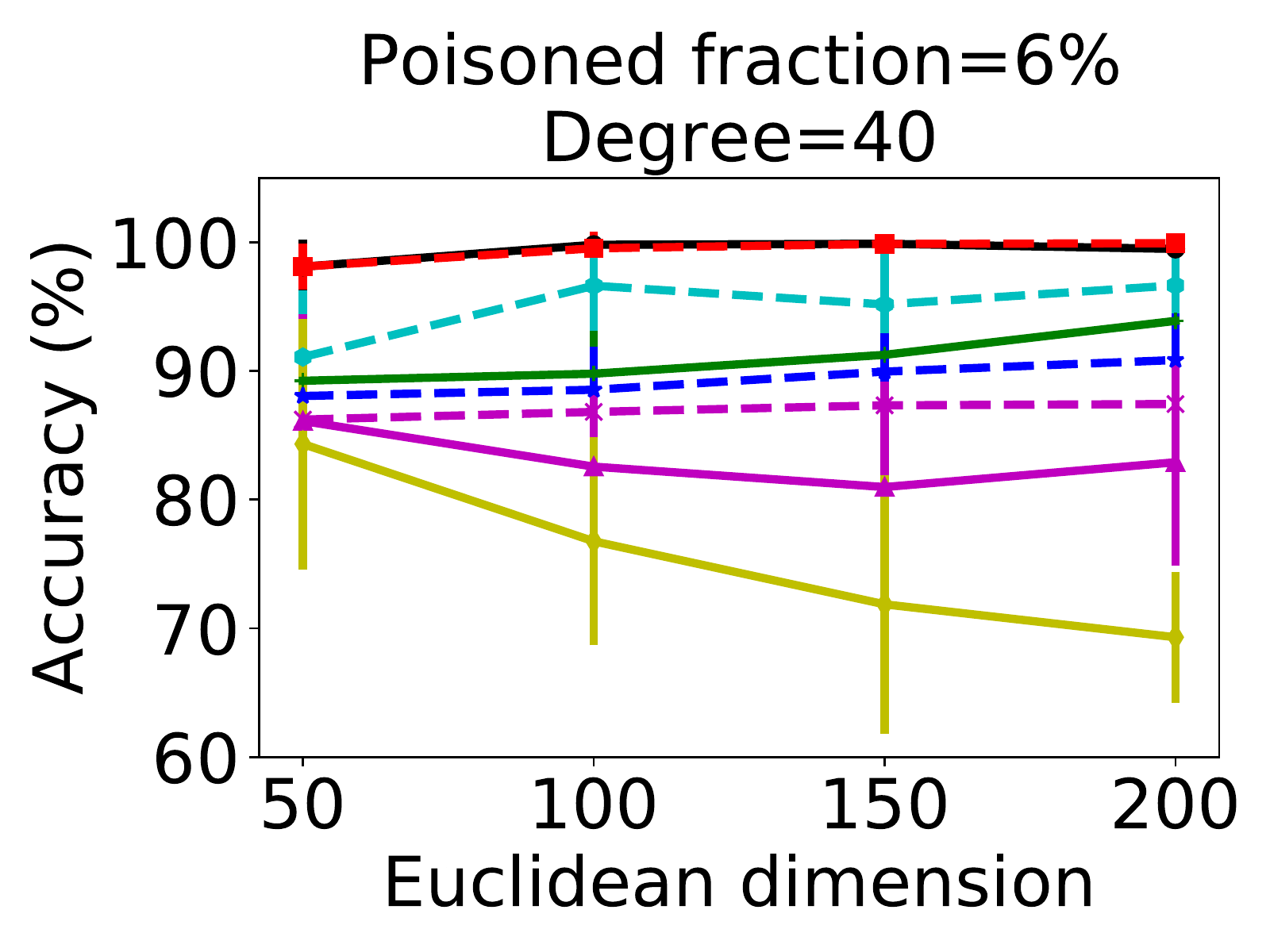}\label{fig-s21}}
	\hspace{0.5in}	
	\subfloat[]{\includegraphics[height=1.2in]{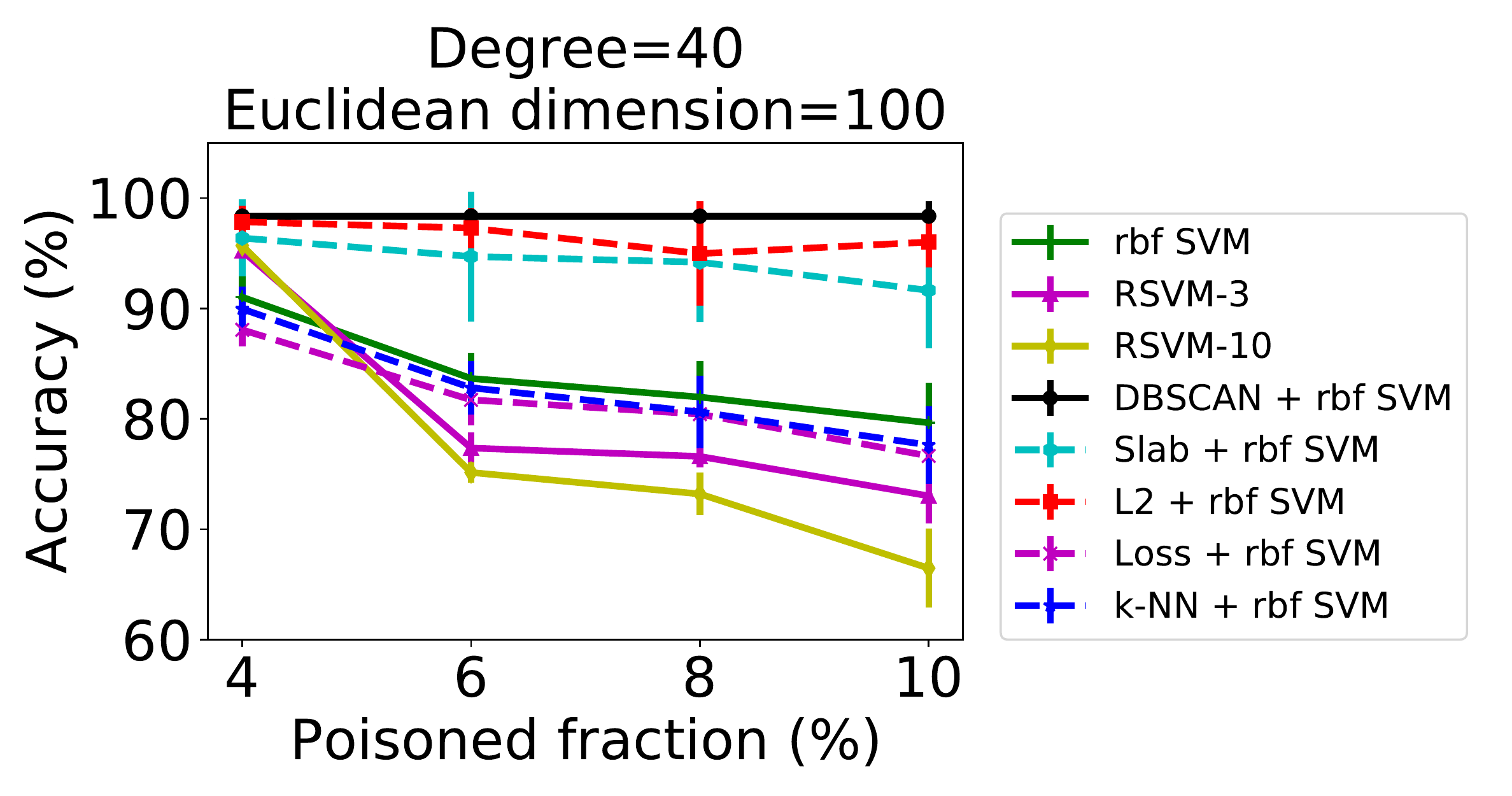}\label{fig-s22}}	
	\vspace{-0.1in}
	
	\subfloat[]{\includegraphics[height=1.2in]{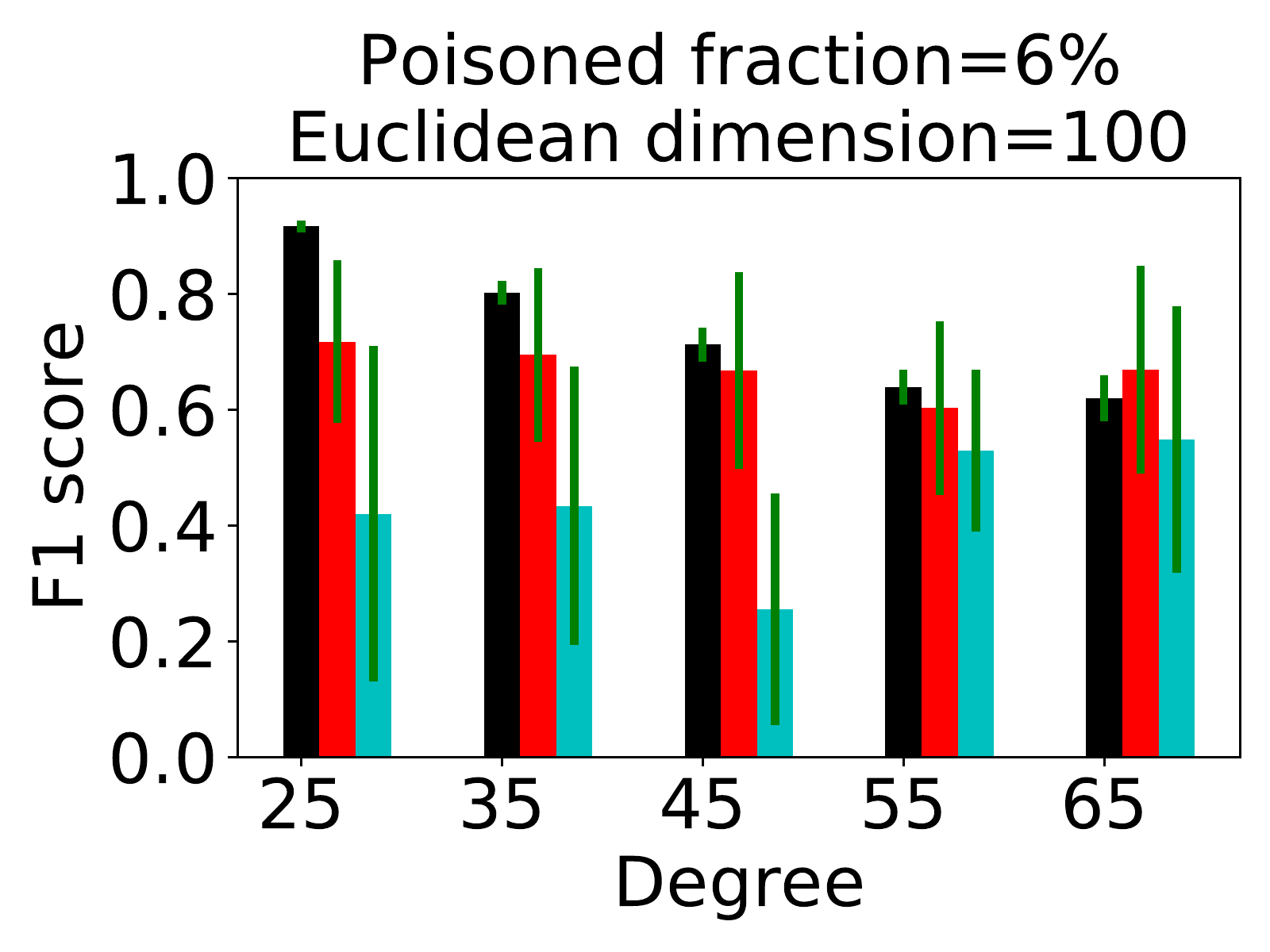}\label{fig-30}}
	\hspace{0.4in}	
	\subfloat[]{\includegraphics[height=1.2in]{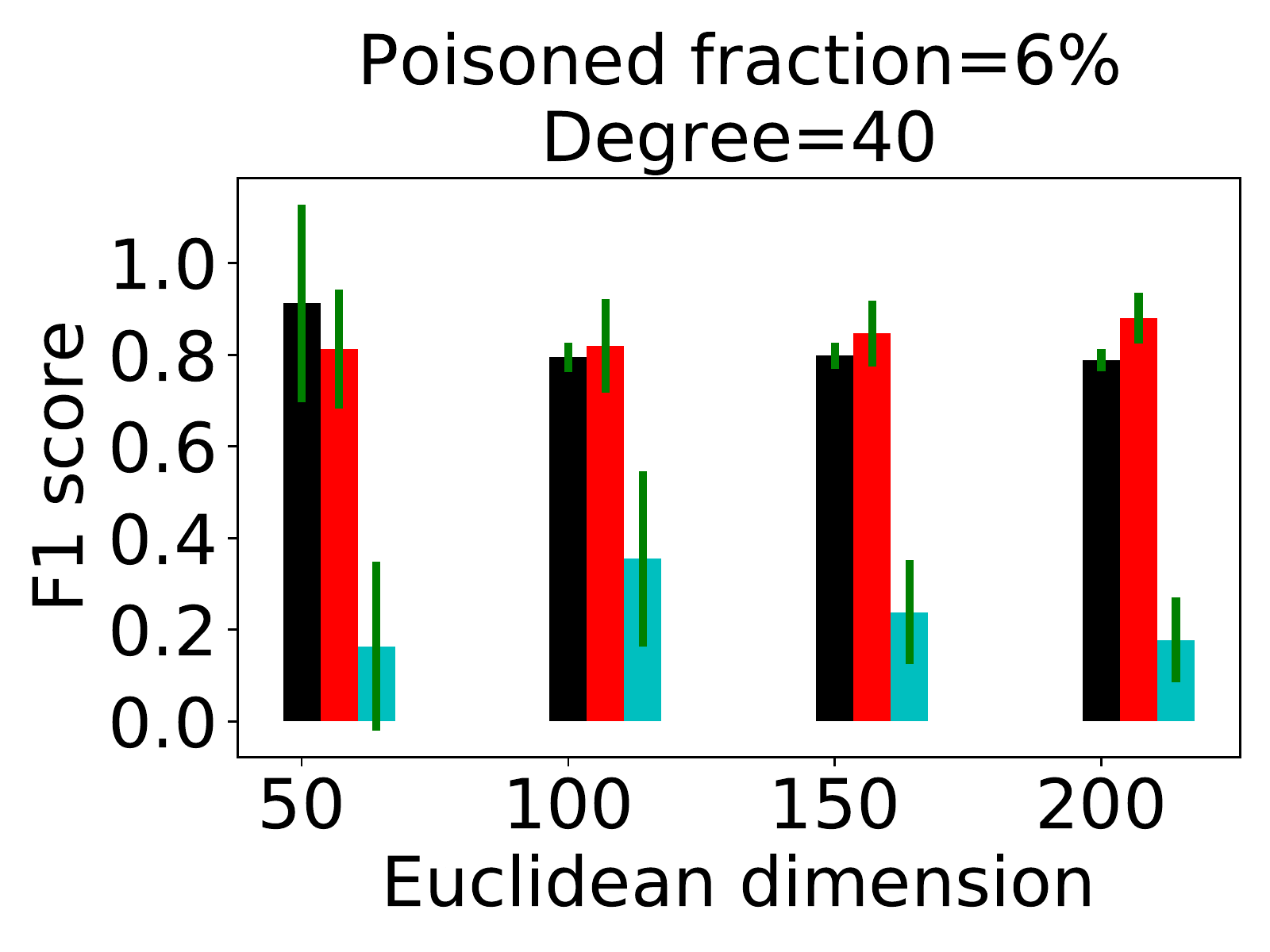}\label{fig-31}}
	\hspace{0.4in}	
	\subfloat[]{\includegraphics[height=1.2in]{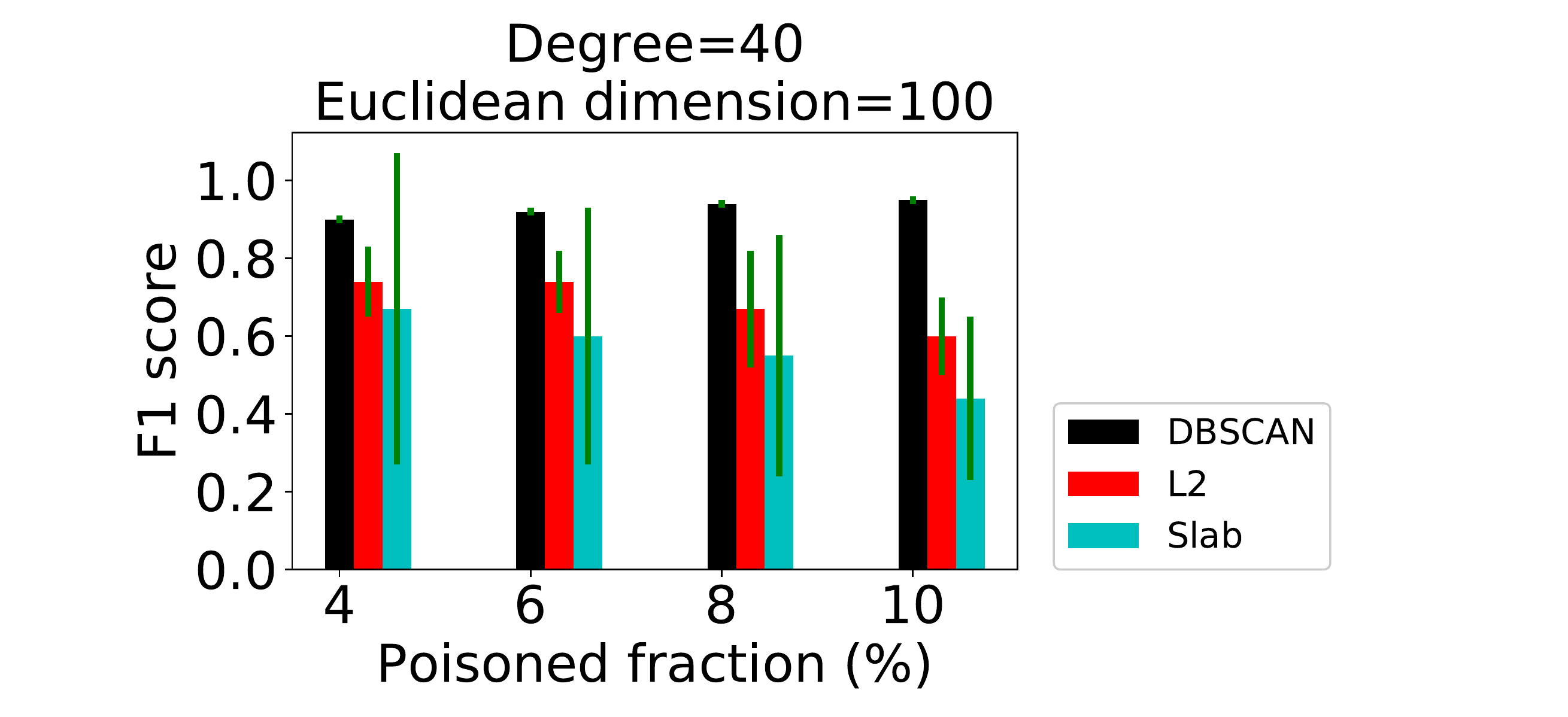}\label{fig-32}}	
	\vspace{-0.1in}
	%	\subfloat[\textsc{Letter}]{\includegraphics[height=1.1in]{mmletterlinearl}\label{fig-s13}}
	\caption{The classification accuracy on the \textsc{Synthetic} datasets of Linear SVM (the first line) and SVM with RBF kernel (the second line) under \textsc{Min-Max} attack. The third line are the average $F_1$ scores. }
	\label{fig-su1}
	\vspace{-0.3in}
\end{figure*}

%\begin{figure}[]
%	\centering
%	\includegraphics[height=1.35in]{f1score}
%	\caption{$F_1$ scores on the \textsc{Synthetic} datasets.}
%	\label{fig-3}
%\end{figure}

%\begin{table}[ht]	
%	\centering
%	% table caption is above the table	
%	\caption{$F_1$ scores.}	
%	\label{tab:ind}       % Give a unique label	
%	% For LaTeX tables use	
%	\begin{tabular}{cccccc}		
%		\hline\noalign{\smallskip}		
%		Poisoned fraction & $4\%$ & $6\%$ & $8\%$ & $10\%$ \\		
%		\noalign{\smallskip}\hline\noalign{\smallskip}		
%		\textsc{DBSCAN} & $0.95$ & $0.87$ & $0.80$ & $0.75$ & $0.73$  \cr	
%		\textsc{L$2$} & $0.69$ & $0.69$ & $0.61$ & $0.66$ & $0.60$  \cr	
%		\textsc{Slab} & $0.43$ & $0.41$ & $0.17$ & $0.19$ & $0.22$  \cr	
%		\textsc{Loss} & $<0.1$ & $<0.1$ & $<0.1$ & $<0.1$ & $<0.1$  \cr	
%		\textsc{Knn} & $<0.1$ & $<0.1$ & $<0.1$ & $<0.1$ & $<0.1$  \cr	
%		\noalign{\smallskip}\hline			
%	\end{tabular}
%	\label{tab-ind}
%\end{table}

%\begin{figure*}[]
%	\centering
%	%	\subfloat[\textsc{Synthetic}]{\includegraphics[height=1.5in]{mmsyn5linear}\label{fig-ex11}}
%	\subfloat[\textsc{Min-Max} attack]{\includegraphics[height=1.35in]{f1score}\label{fig-f30}}
%	\hspace{0.4in}
%	\subfloat[\textsc{ALFA} attack]{\includegraphics[height=1.35in]{f1score2}\label{fig-f31}}
%	\caption{$F_1$ scores.}
%	\label{fig-f3}
%\end{figure*}

\begin{figure*}[]
	\centering
	\vspace{-0.4in}
%	\subfloat[\textsc{Synthetic}]{\includegraphics[height=1.5in]{mmsyn5linear}\label{fig-ex11}}
	\subfloat[\textsc{Letter}]{\includegraphics[height=1.2in]{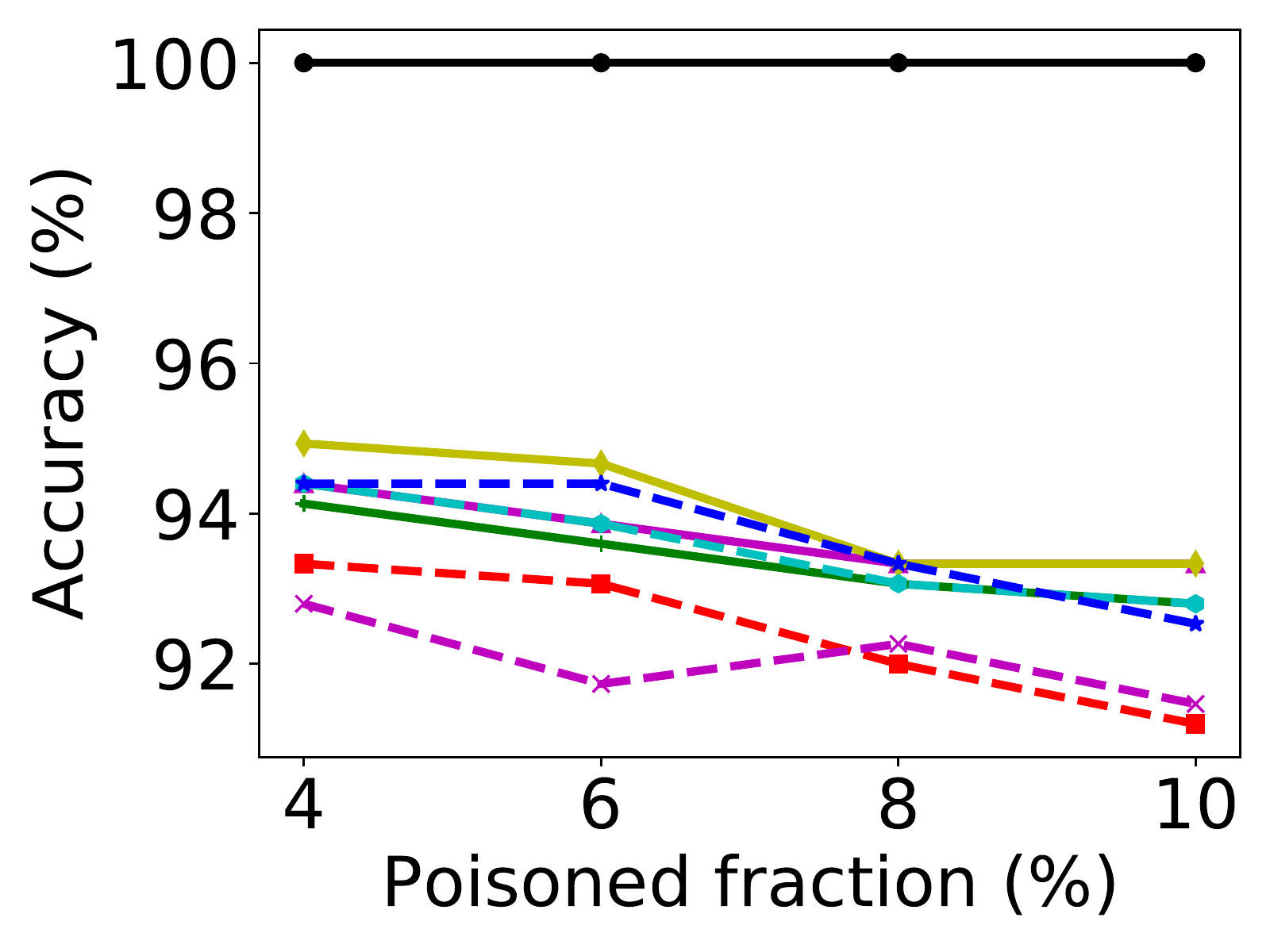}\label{fig-ex12}}
		\hspace{0.4in}	
	\subfloat[\textsc{Mushrooms}]{\includegraphics[height=1.2in]{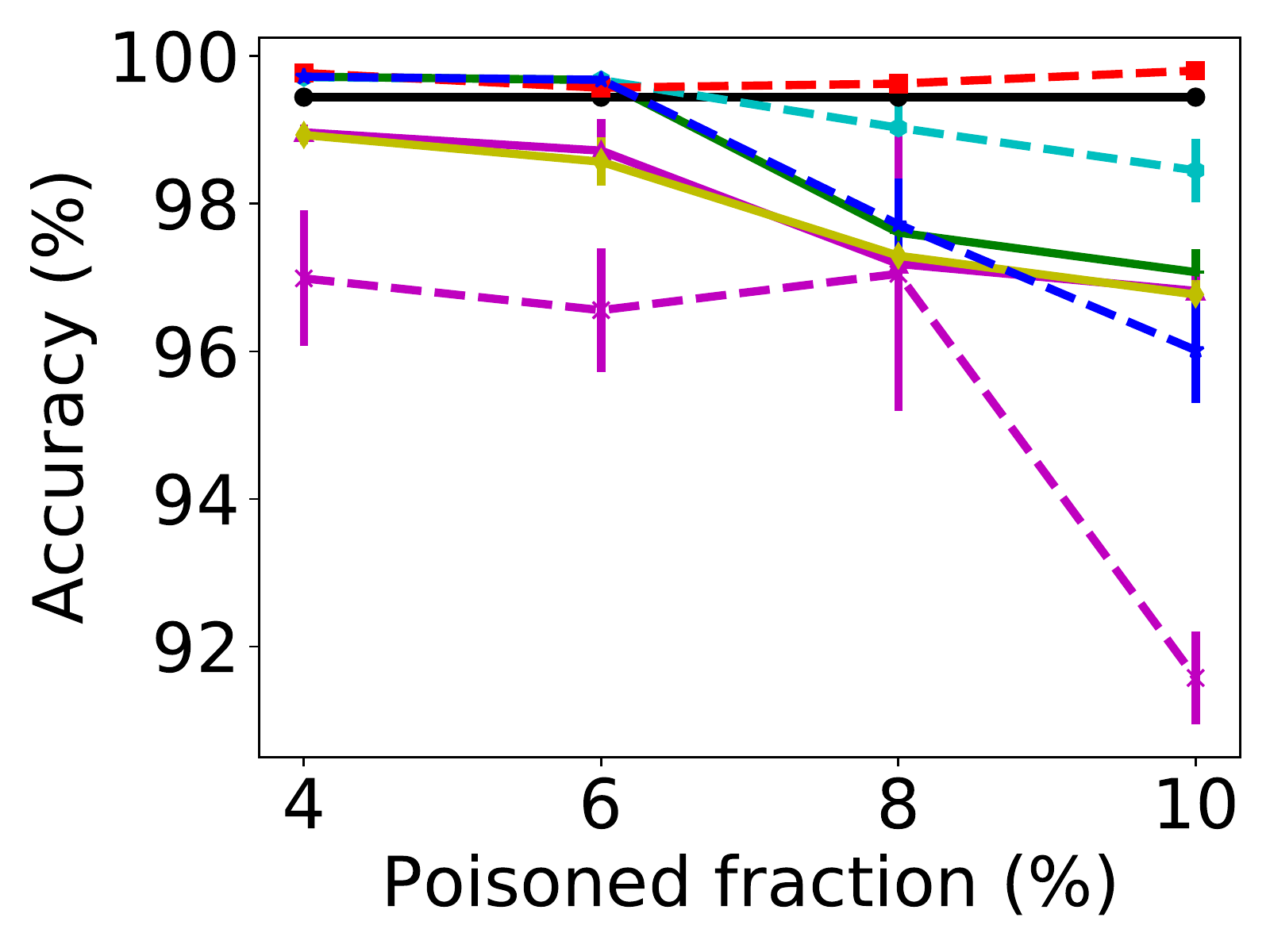}\label{fig-ex13}}
		\hspace{0.4in}	
	\subfloat[\textsc{Satimage}]{\includegraphics[height=1.2in]{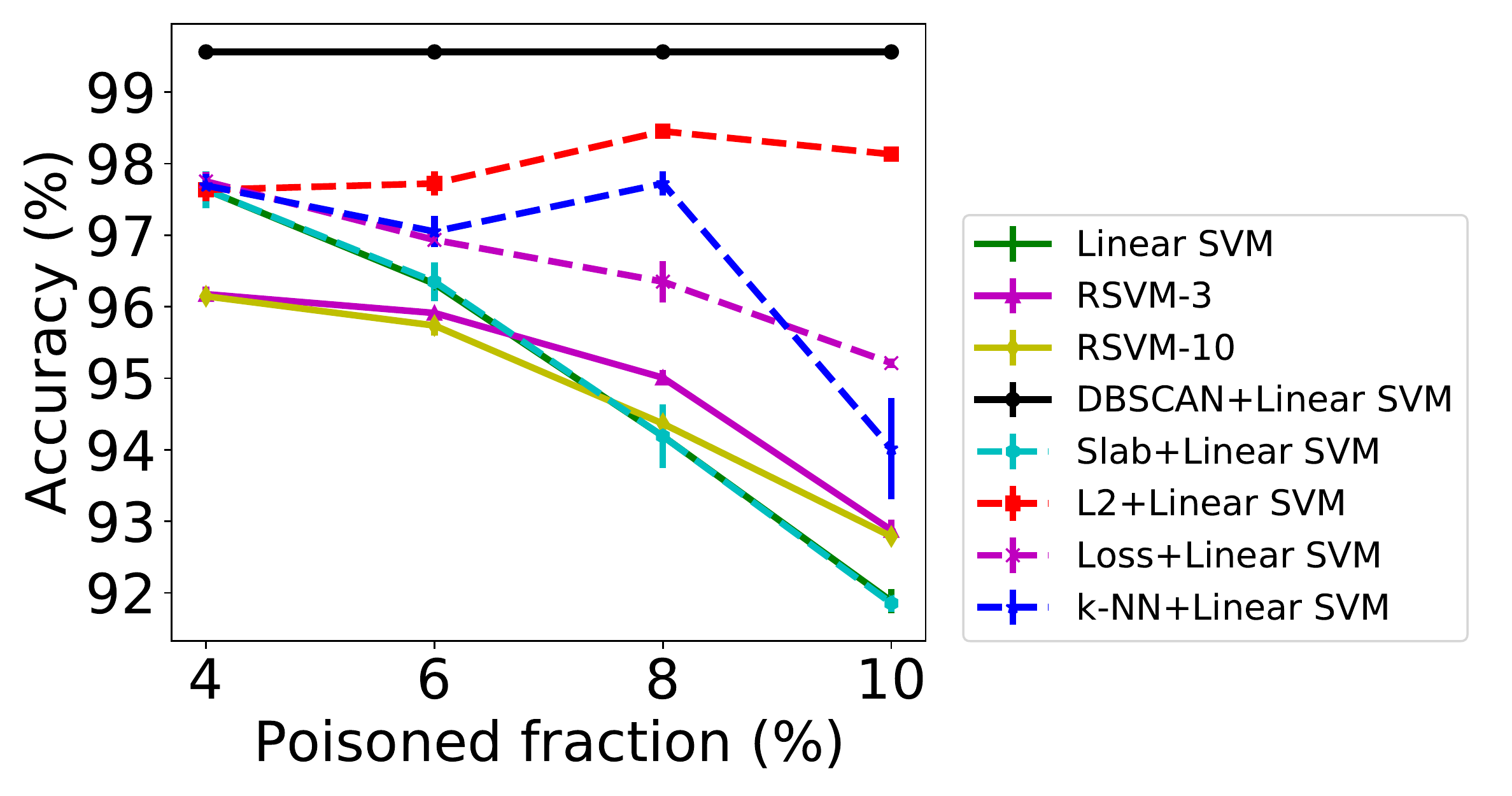}\label{fig-ex14}}
	\vspace{-0.1in}
	
	\subfloat[\textsc{Letter}]{\includegraphics[height=1.2in]{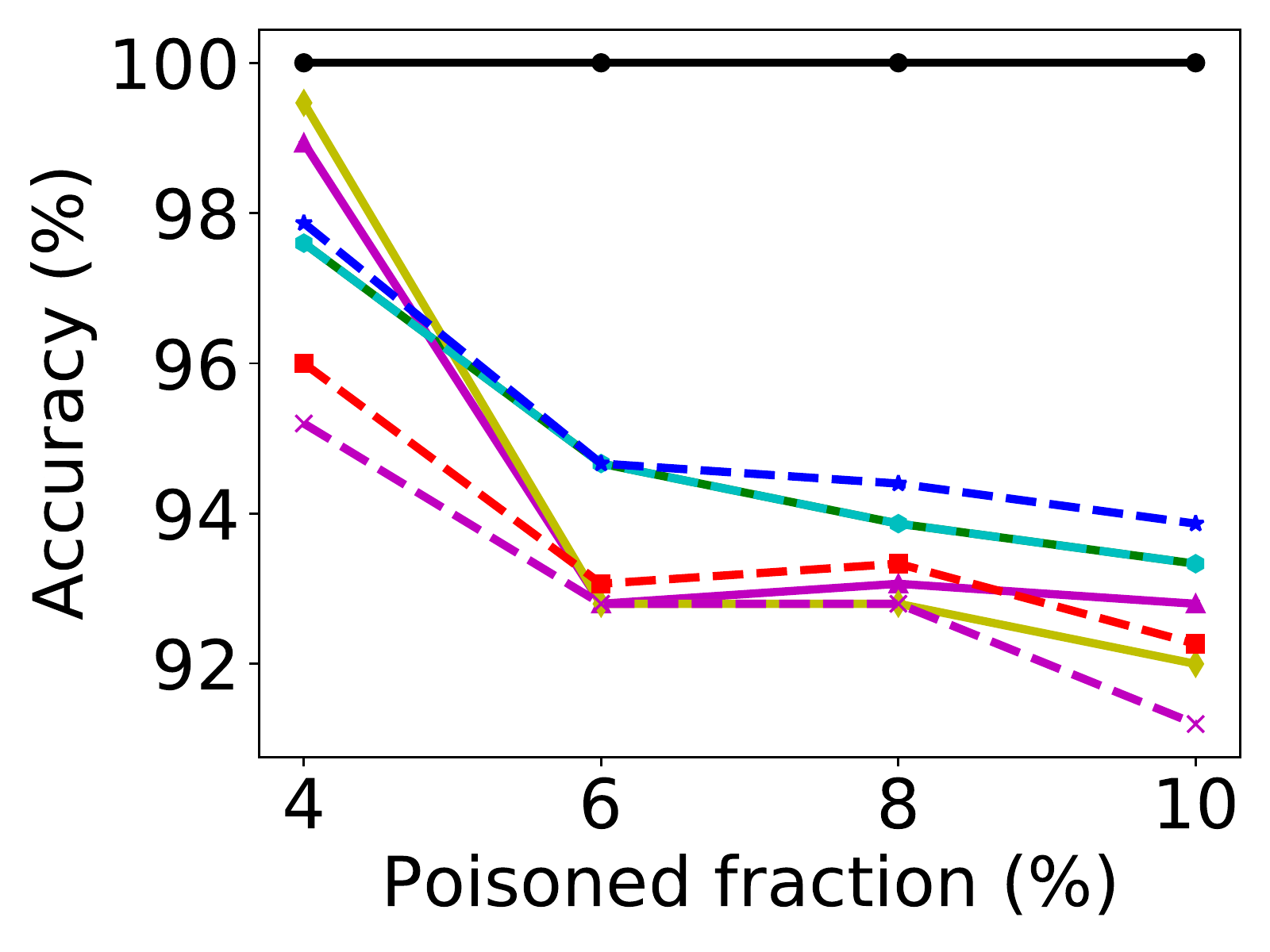}\label{fig-ex22}}
		\hspace{0.4in}	
	\subfloat[\textsc{Mushrooms}]{\includegraphics[height=1.2in]{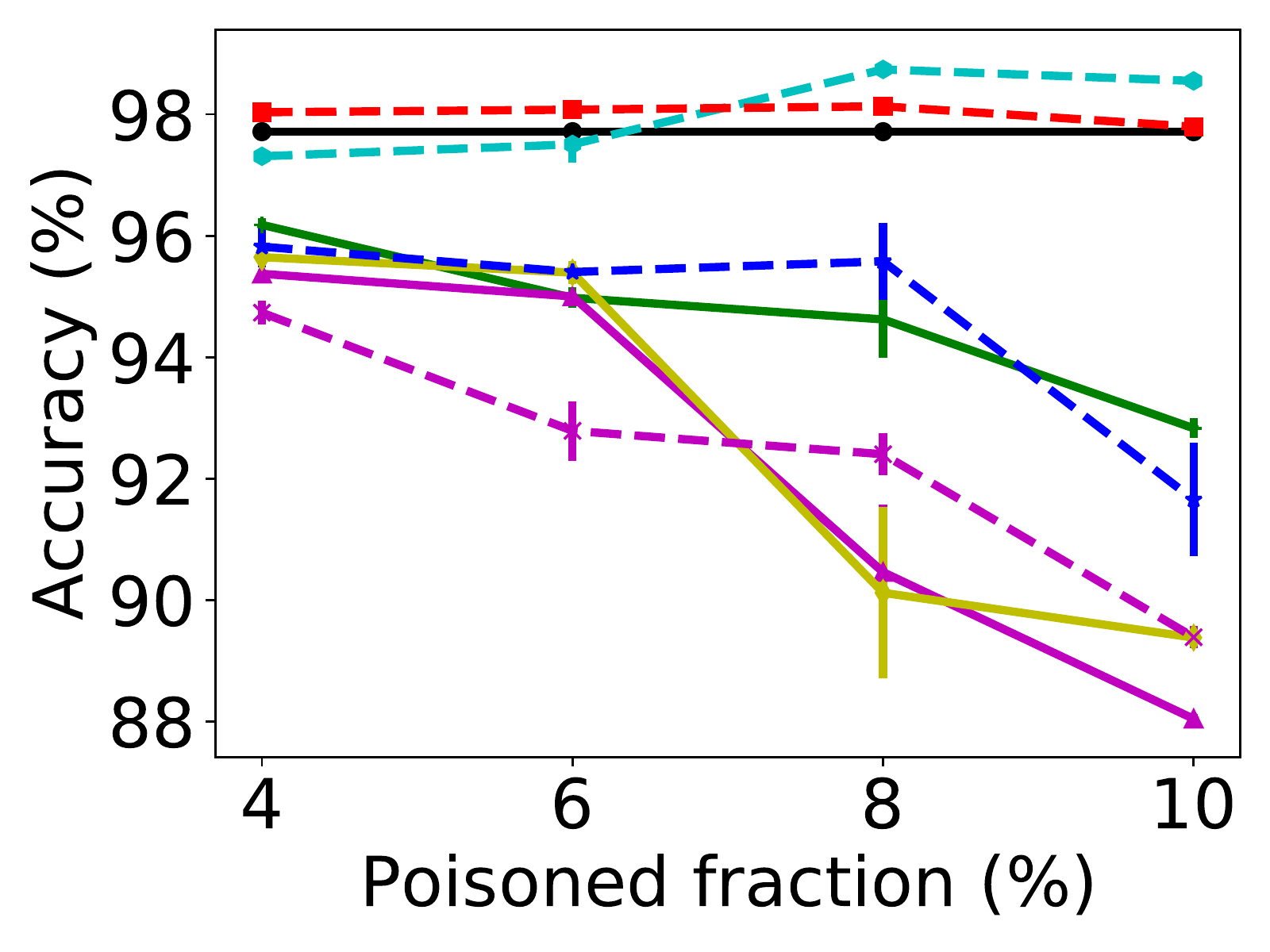}\label{fig-ex23}}
		\hspace{0.4in}	
	\subfloat[\textsc{Satimage}]{\includegraphics[height=1.2in]{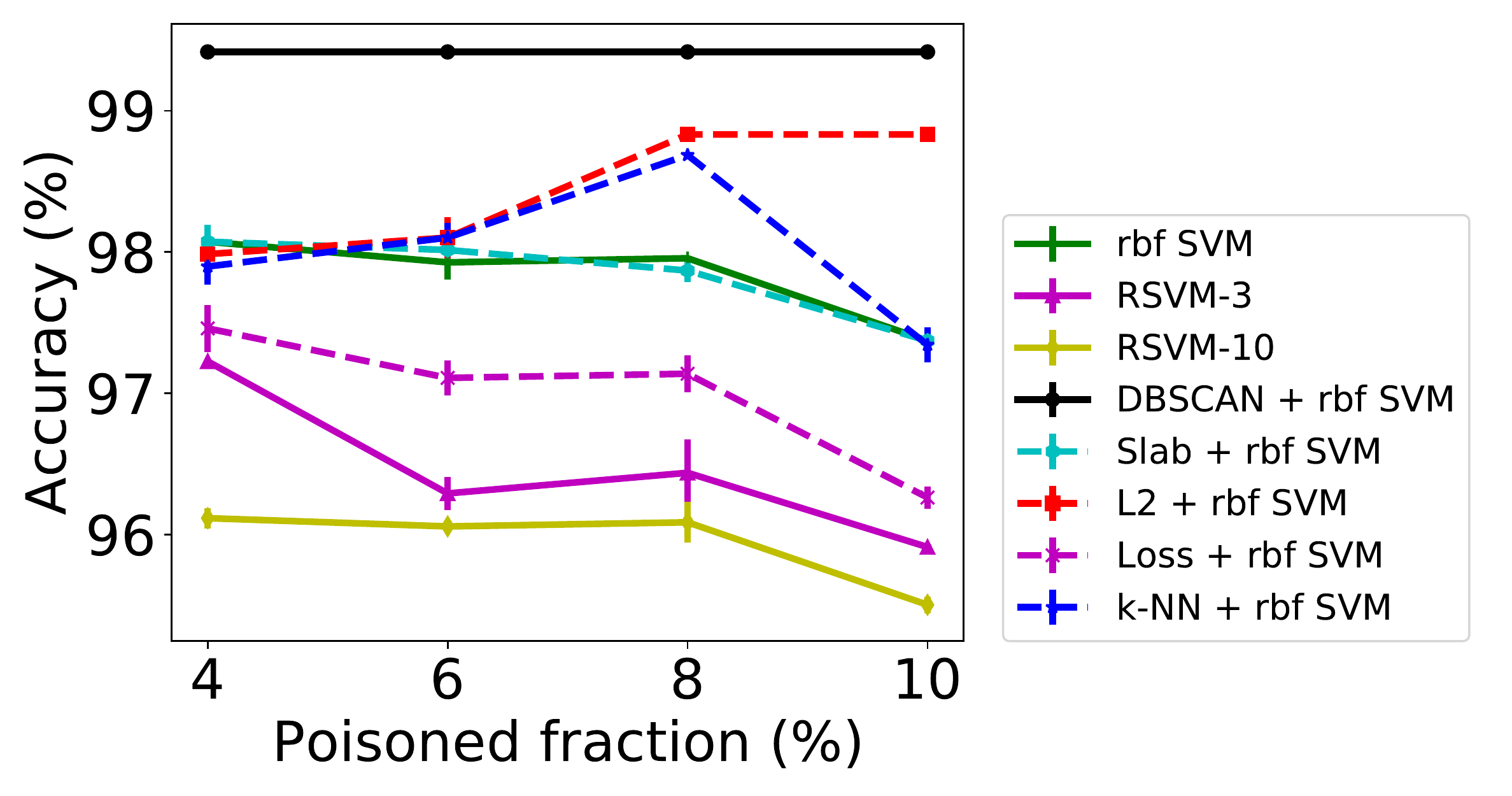}\label{fig-ex24}}
	\caption{The  classification accuracy on the real datasets of linear SVM (the first line) and SVM with RBF kernel (the second line) under \textsc{Min-Max} attack.}
	\label{fig-ex1}
	\vspace{-0.3in}
\end{figure*}

%\begin{figure*}[]
%	\centering
%	\subfloat[]{\includegraphics[height=1.2in]{mmsyn65ind5linear}\label{fig-s10}}
%	\hspace{0.5in}
%	\subfloat[]{\includegraphics[height=1.2in]{mmsyn200eud5linear}\label{fig-s11}}
%		\hspace{0.5in}	
%	\subfloat[]{\includegraphics[height=1.2in]{mmsynlinearl}\label{fig-s12}}	
%%	\subfloat[\textsc{Letter}]{\includegraphics[height=1.1in]{mmletterlinearl}\label{fig-s13}}
%	\caption{The classification accuracy of Linear SVM under \textsc{Min-Max} attack.}
%	\label{fig-su1}
%\end{figure*}
%
%
%\begin{figure*}[]
%	\centering
%	\subfloat[]{\includegraphics[height=1.2in]{mmsyn65ind5rbf}\label{fig-s20}}
%		\hspace{0.5in}	
%	\subfloat[]{\includegraphics[height=1.2in]{mmsyn200eud5rbf}\label{fig-s21}}
%		\hspace{0.5in}	
%	\subfloat[]{\includegraphics[height=1.2in]{mmsynrbfl}\label{fig-s22}}	
%%	\subfloat[\textsc{Letter}]{\includegraphics[height=1.1in]{mmletterrbfl}\label{fig-s23}}
%	\caption{The classification accuracy of SVM with RBF kernel under \textsc{Min-Max} attack.}
%	\label{fig-su2}
%\end{figure*}

%We compare several defenses including the DBSCAN approach, and study the trends of their classification accuracies with varying the poisoned fraction, the intrinsic dimensionality, and the Euclidean dimensionality. 
All the experiments were repeated $20$ times on a Windows $10$ workstation equipped with an Intel core $i5$-$8400$ processor and $8$GB RAM. 
To generate the poisoning attacks, we use the \textbf{\textsc{Min-Max}} attack from \citep{DBLP:journals/corr/abs-1811-00741} and the adversarial label-flipping attack \textbf{\textsc{ALFA}} from ALFASVMLib~\citep{DBLP:journals/ijon/XiaoBNXER15}. 
We evaluate the defending performances of the basic SVM algorithms and several different defenses by using their publicly available implementations.
\vspace{-0.1in}
\begin{enumerate}
	\item  We consider both the cases that not using and using kernel. For SVM without kernel, we directly use  \textbf{\textsc{linear SVM}} as the basic SVM algorithm; for SVM with kernel, we consider RBF kernel (\textbf{\textsc{rbf SVM}}).  Both the implementations are from~\citep{journals/tist/ChangL11}.  
	
	\item The recently proposed robust SVM algorithm \textbf{\textsc{RSVM-$S$}} based on the rescaled hinge loss function~\citep{DBLP:journals/pr/XuCHP17}.  The parameter ``$S$'' indicates the iteration number of the half-quadratic optimization ({\em e.g.,} we set $S=3$ and $10$ following their paper's setting). The algorithm also works fine when using a kernel. 
	
	%	\item The fast and robust twin support vector machine \textbf{\textsc{FRTSVM}}~\cite{DBLP:journals/corr/abs-1711-05406}; 

	\item The \textbf{\textsc{DBSCAN}} method~\citep{schubert2017dbscan} implemented as Remark~\ref{rem-dbscan} (\rmnum{2}). We set $\mathtt{MinPts}=5$ (our empirical study finds that  the difference is minor within the range $[3,10]$).  
	\item The data sanitization defenses from \citep{DBLP:journals/corr/abs-1811-00741} based on the spatial distribution of input data, which include \textbf{\textsc{Slab}}, \textbf{\textsc{L2}}, \textbf{\textsc{Loss}}, and \textbf{\textsc{k-NN}}.  
\end{enumerate}
\vspace{-0.1in}
%We also consider the following data sanitization defenses from \cite{DBLP:journals/corr/abs-1811-00741}:
%\saveenum
%\begin{enumerate}\resetenum
%	\item \textbf{\textsc{Slab}}:  first project the points onto the line between 
%	the class centroids, and then remove the points that are too far from the class centroids; 
	
%	\item  \textbf{\textsc{L2}}: remove the points that are far from their class centroids in $L_2$ distance; 

%	\item \textbf{\textsc{Loss}}: discard the points that are not well fit by a model trained (without any data sanitization) on the full dataset; 
	
%	\item \textbf{\textsc{k-NN}}: remove the points that are far from their $k$ nearest neighbors (we let $k=5$ as their paper's setting); 
	%	\item \textbf{\textsc{SVD}}: assume that the clean data lies in some low-rank subspace, and the 
	%	poisoned data therefore should have a large component out of this subspace~\cite{DBLP:conf/imc/RubinsteinNHJLRTT09};
	
%\end{enumerate}

%Our experiments have two parts: using the  
For the data sanitization defenses, we run them on the poisoned data in the original input space; then, apply the basic SVM algorithm, \textsc{linear SVM} or \textsc{rbf SVM} (if using RBF kernel),  on the cleaned data to compute their final solutions.

\vspace{-0.15in}
\begin{table}[ht]	
	\centering
	\caption{Datasets}	
	\label{tab:1}       % Give a unique label	
	\vspace{-0.1in}
	\begin{tabular}{lcc}		
		\hline\noalign{\smallskip}		
		Dataset & Size & Dimension  \\		
		\noalign{\smallskip}\hline\noalign{\smallskip}		
		\textsc{Synthetic} & $10000$ & $50$-$200$  \\	
		\textsc{letter} & $1520$ & $16$ 	\\
		\textsc{mushrooms} & $8124$ & $112$  \\
		\textsc{satimage} & $2236$ & $36$  \\
		\noalign{\smallskip}\hline			
	\end{tabular}
	\label{tab-1}
	\vspace{-0.1in}
\end{table}

\textbf{Datasets.} 
%We consider both the synthetic and real datasets in our experiments. For each synthetic dataset, we generate two manifolds in $\mathbb{R}^{d}$, where $d$ is between $50$ and $200$, and each manifold is represented by a random polynomial function with degree ranging from $25$ to $65$. Note that it is challenging to achieve the exact doubling dimensions of the datasets,  so we use the degree of the polynomial function as a ``rough  indicator'' for the doubling dimension (the higher the degree, the larger the doubling dimension). In each of the manifolds, we sample $5000$ points; specifically, the data is randomly partitioned into $30\%$ and $70\%$ respectively for training and testing. 
We consider both the synthetic and real-world datasets in our experiments. 
For each synthetic dataset, we generate two manifolds in $\mathbb{R}^{d}$, and each manifold is represented by a random polynomial function with degree $d'$ (the values of $d$ and $d'$ will be varied in the experiments). Note that it is challenging to achieve the exact doubling dimensions of the datasets,  and thus we use the degree of the polynomial function as a ``rough  indicator'' for the doubling dimension (the higher the degree, the larger the doubling dimension). In each of the manifolds, we randomly sample $5000$ points; the data is randomly partitioned into $30\%$ and $70\%$ respectively for training and testing, and we report the classification accuracy on the test data.
We also consider three real-world datasets from \citep{journals/tist/ChangL11}. 
%Note that the \textsc{letter} and \textsc{satimage} datasets are for multi-classification tasks, so we only select two classes from each dataset.
%Notice that for letter dataset, we only choose two classes, y and z. 
The details are shown in Table~\ref{tab-1}.

\textbf{Results.}
First, we study the influence from the intrinsic dimensionality. We set the Euclidean dimensionality $d$ to be $100$ and vary the polynomial function's degree $d'$ from $25$ to $65$ in Figure~\ref{fig-s10} and \ref{fig-s20}. Then, we fix the degree $d'$ to be $40$ and vary the Euclidean dimensionality $d$ in Figure~\ref{fig-s11} and \ref{fig-s21}. We can observe that the accuracies of most methods dramatically decrease when the degree $d'$ (intrinsic dimension) increases, and the influence from the intrinsic dimension is more significant than that from the Euclidean dimension, which is in agreement with our theoretical analysis. 

%For some methods ({\em e.g.,} ) the influence from the Euclidean dimensionality is small for most methods.

We also study their classification performances under different poisoned fraction in Figure~\ref{fig-s12} and \ref{fig-s22}. We can see that all the defenses yield lower accuracies when the poisoned fraction increases, while the performance of \textsc{DBSCAN} keeps much more stable compared with other defenses. 
Moreover, we calculate the widely used $F_1$ scores from the sanitization defenses for identifying the outliers.  \textsc{Loss} and \textsc{k-NN} both yield very low $F_1$ scores ($<0.1$); that means they are not quite capable to identify the real poisoning data items. The $F_1$ scores yielded by \textsc{DBSCAN}, \textsc{L$2$} and \textsc{Slab} are shown in Figure~\ref{fig-30}-\ref{fig-32}, where \textsc{DBSCAN} in general outperforms the other two sanitization defenses for most cases.

We also perform the experiments on the real  datasets under \textsc{Min-Max} attack  with the poisoned fraction ranging from $4\%$ to $10\%$.  
%illustrates the results of linear SVM under Min-Max attack, Figure~\ref{fig-ex2} illustrates the results of SVM with RBF kernel under Min-Max attack, Figure~\ref{fig-ex3} shows the results of linear SVM under ALFA attack and the results of SVM with RBF kernel under ALFA attack are shown in Figure~\ref{fig-ex4}. 
The experimental results (Figure~\ref{fig-ex1}) reveal the similar trends as the results for the synthetic datasets, and \textsc{DBSCAN} keeps considerably better performance compared with other defenses. Due to the space limit, the results under \textsc{ALFA} attack are shown in Appendix.

 \vspace{-0.05in}

\section{Discussion}
\vspace{-0.05in}
In this paper, we study two different strategies for protecting SVM against poisoning attacks. 
%
%% To achieve the adversarial-resilience, the defense can be formulated as a combinatorial optimization problem called ``SVM with outliers''. We show for the first time that even the simplest hard-margin one-class SVM with outliers is NP-complete, and has no fully PTAS unless P$=$NP. We then focus on the data sanitization defense. Under the assumption that the original input data (before poisoning attack) are drawn from the domains with low doubling dimensions, we provide the lower bound of the data size to ensure that  the DBSCAN can correctly identify the poisoning samples. 
%%
%%We leave the detailed experimental results on the synthetic and real datasets to our supplement. We compare several defenses including the DBSCAN and robust SVM methods, and study the trends of their classification accuracies with varying three values: the poisoned fraction, the intrinsic dimensionality, and the Euclidean dimensionality. All the experimental results were obtained by using publicly available implementations on a Windows $10$ workstation equipped with an Intel core $i5$-$8265U$ processor and $8$GB RAM.
%%
%%% ; the algorithms are implemented in Matlab R2019a.   
%%
%%
%
%
We also have several open questions to study in future. 
For example, what about the complexities of other machine learning problems under the adversarially-resilient formulations as Definition~\ref{def-svm}? %\cite{DBLP:journals/algorithmica/MountNPSW14} proved that it is impossible to achieve even an approximate solution for the linear regression with outliers problem within polynomial time under the conjecture of {\em the hardness of affine degeneracy}~\cite{DBLP:journals/dcg/EricksonS95}, if the dimensionality $d$ is not fixed. 
 For many  other adversarial machine learning problems,  the study on their complexities is still in its infancy.

\newpage

\bibliography{stability}
\newpage
\appendix
%\section{More Experimental Results}
%\vspace{-0.2in}
\begin{figure*}[h]
	\centering
%	\subfloat[\textsc{Synthetic}]{\includegraphics[height=1.5in]{alfasyna5linear}\label{fig-ex31}}
	\subfloat[\textsc{Letter}]{\includegraphics[height=1.2in]{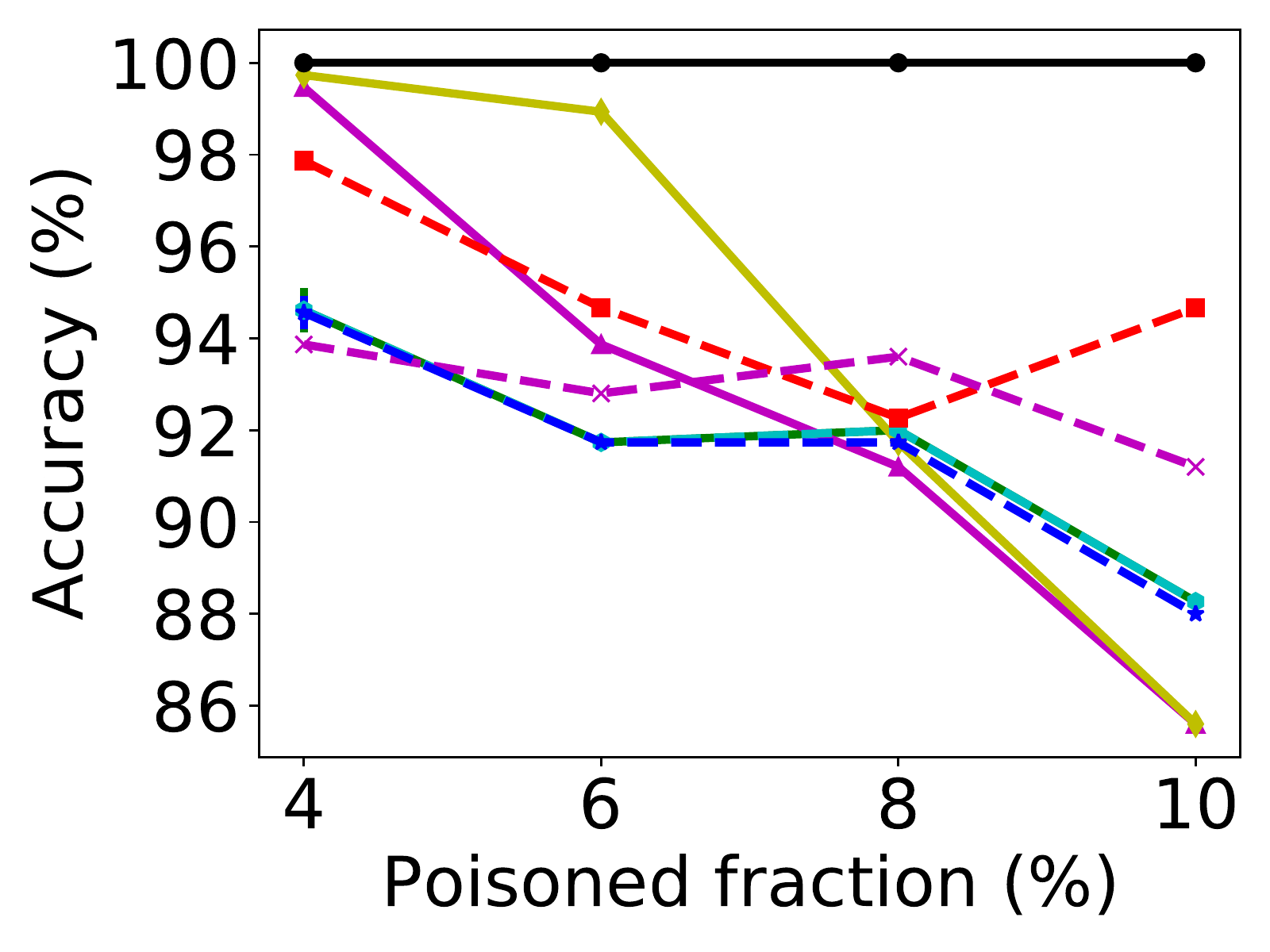}\label{fig-ex32}}
		\hspace{0.4in}	
	\subfloat[\textsc{Mushrooms}]{\includegraphics[height=1.2in]{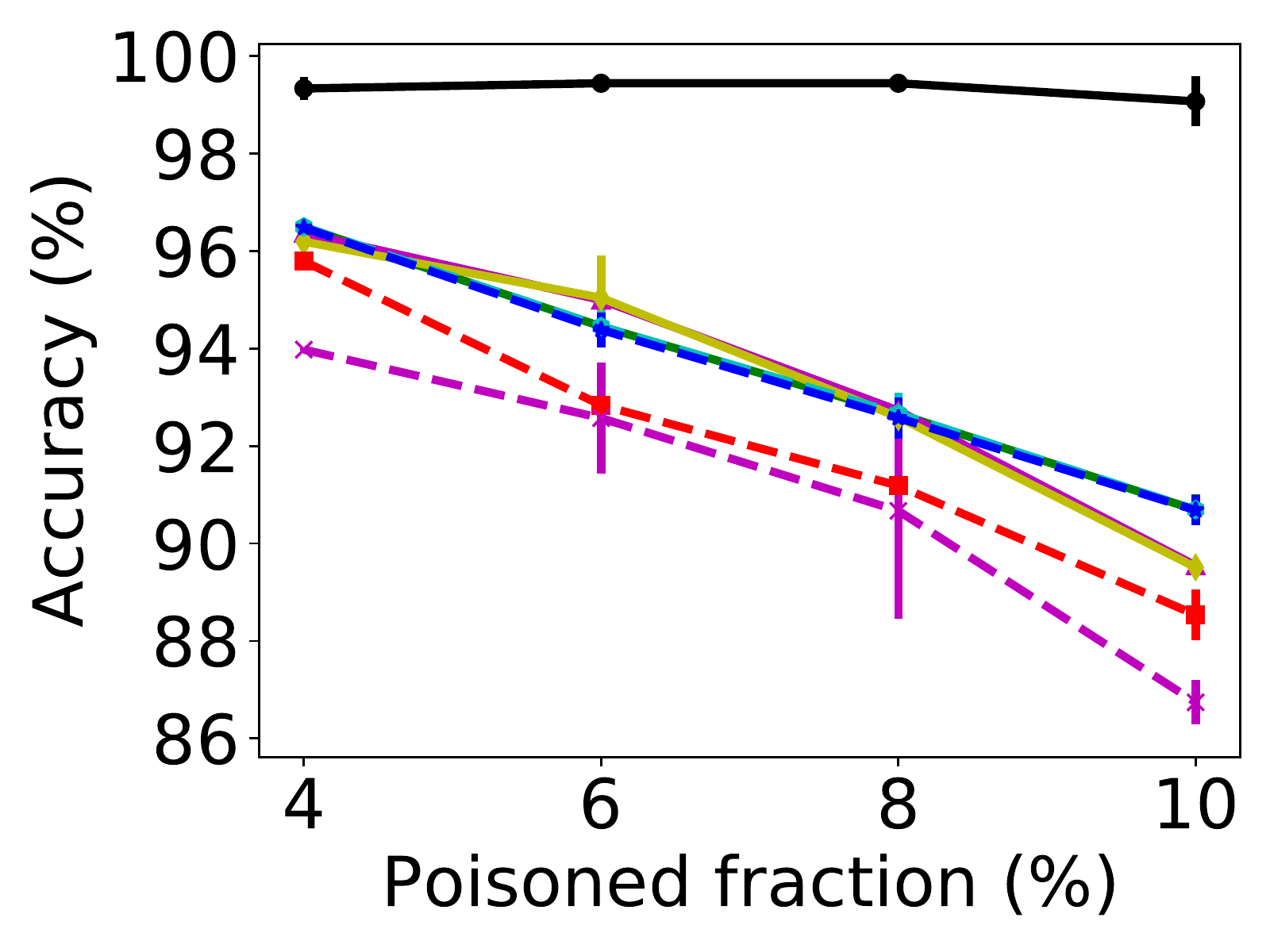}\label{fig-ex33}}
		\hspace{0.4in}	
	\subfloat[\textsc{Satimage}]{\includegraphics[height=1.2in]{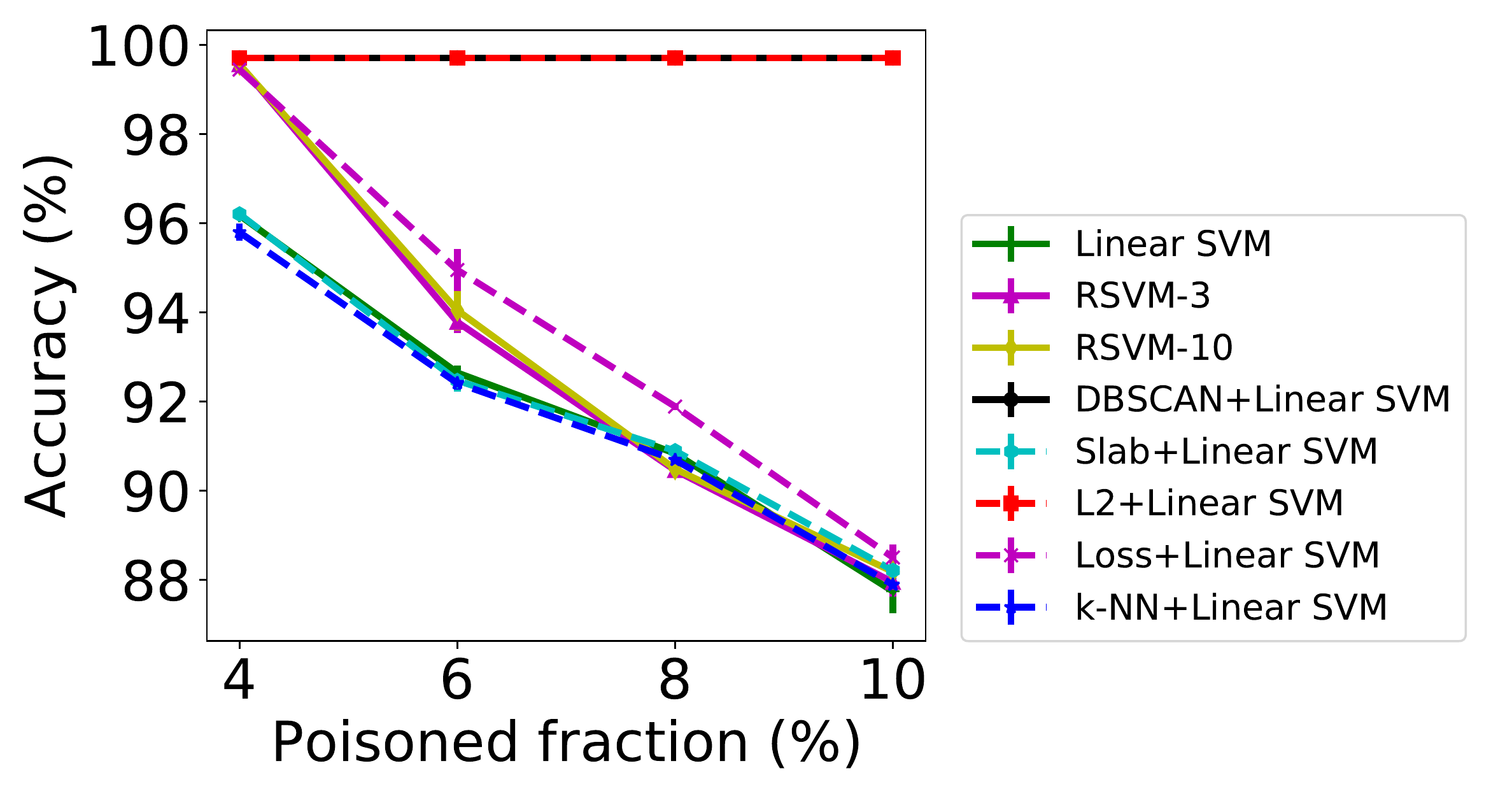}\label{fig-ex34}}
	\caption{The  classification accuracy of linear SVM under \textsc{ALFA} attack.}
	\label{fig-app-ex3}
\end{figure*}

\begin{figure*}[]
	\centering
%	\subfloat[\textsc{Synthetic}]{\includegraphics[height=1.5in]{alfasyna5rbf}\label{fig-ex41}}
	\subfloat[\textsc{Letter}]{\includegraphics[height=1.2in]{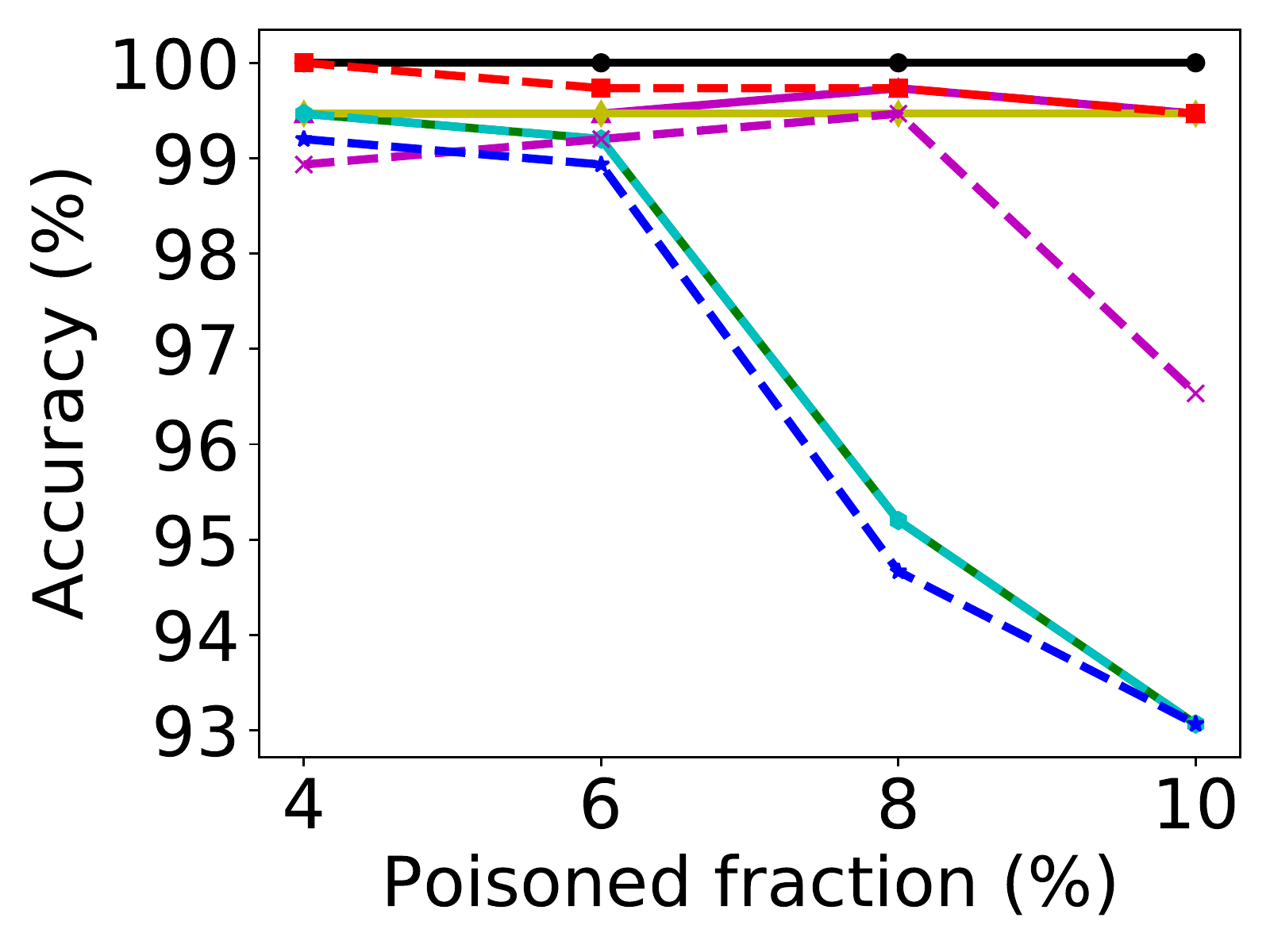}\label{fig-ex42}}
		\hspace{0.4in}	
	\subfloat[\textsc{Mushrooms}]{\includegraphics[height=1.2in]{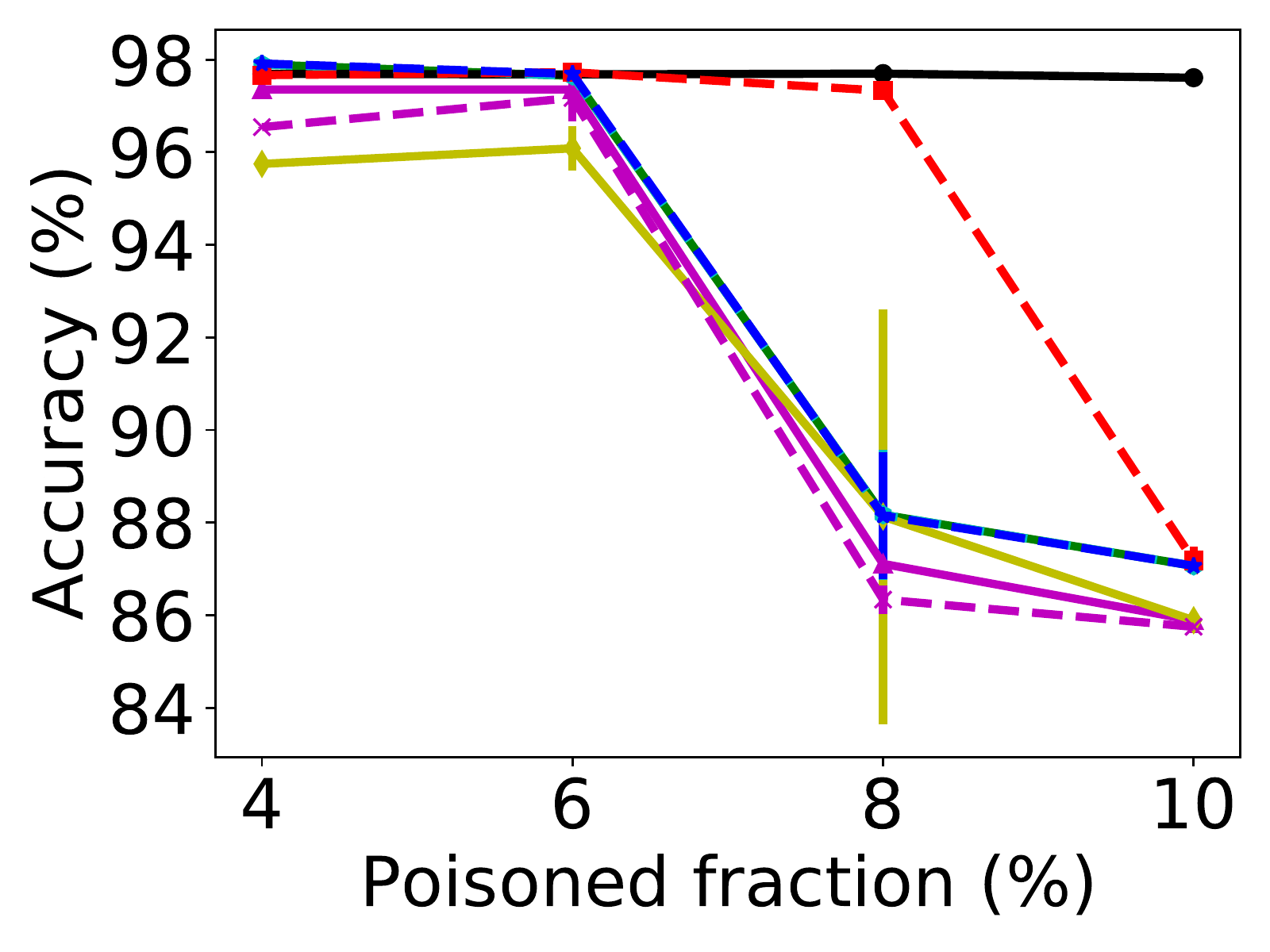}\label{fig-ex43}}
		\hspace{0.4in}	
	\subfloat[\textsc{Satimage}]{\includegraphics[height=1.2in]{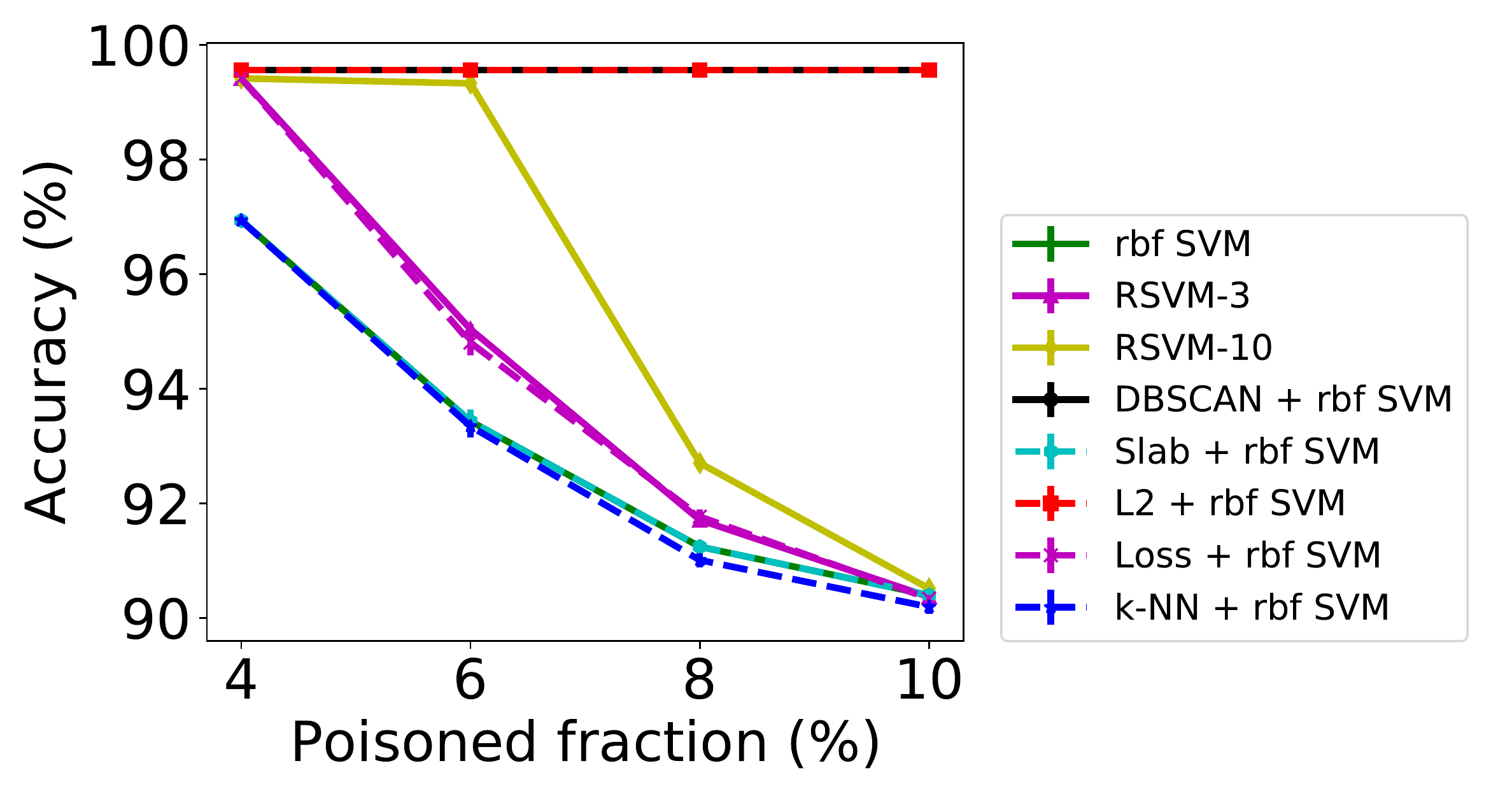}\label{fig-ex44}}
	\caption{The  classification accuracy of SVM with RBF kernel under \textsc{ALFA} attack.}
	\label{fig-app-ex4}
\end{figure*}

\begin{table*}[ht]  	
	\centering  
%	\fontsize{6.5}{8}\selectfont  
	\begin{threeparttable}  
		\caption{$F_1$ scores on \textsc{Letter} dataset. }  
		\label{tab:2}  
		\begin{tabular}{ccccccccc}  
			\toprule  
%			\multirow{2}{*}{Method}&  
			\multicolumn{4}{c}{\textsc{ALFA}}&\multicolumn{4}{c}{\textsc{Min-Max}}\cr 
			\cmidrule(lr){2-5} \cmidrule(lr){6-9}  
			 & $4\%$ & $6\%$ & $8\%$ & $10\%$ & $4\%$ & $6\%$ & $8\%$ & $10\%$\cr
			\midrule  
			\textsc{DBSCAN} & $\bf{1.00}$ & $\bf{1.00}$ & $\bf{1.00}$ & $\bf{1.00}$ & $\bf{1.00}$ & $\bf{1.00}$ & $\bf{1.00}$ & $\bf{1.00}$ \cr		
			\textsc{Slab} & $<0.1$ & $<0.1$ & $<0.1$ & $<0.1$ & $<0.1$ & $<0.1$ & $<0.1$ & $<0.1$ \cr
			\textsc{L$2$} & $0.40$ & $0.46$ & $0.39$ & $0.60$ & $<0.1$ & $<0.1$ & $<0.1$ & $<0.1$ \cr	
			\textsc{Loss} & $<0.1$ & $0.54$ & $0.70$ & $0.59$ & $<0.1$ & $<0.1$ & $0.19$ & $<0.1$ \cr	
			\textsc{Knn} & $<0.1$ & $<0.1$ & $<0.1$ & $<0.1$ & $<0.1$ & $<0.1$ & $<0.1$ & $<0.1$ \cr 
			\bottomrule  
		\end{tabular}  
	\end{threeparttable} 
	\label{tab-2} 
\end{table*}  

\begin{table*}[ht]  	
	\centering  
	%	\fontsize{6.5}{8}\selectfont  
	\begin{threeparttable}  
		\caption{$F_1$ scores on \textsc{Mushroom} dataset. }  
		\label{tab:3}  
		\begin{tabular}{ccccccccc}  
			\toprule  
			%			\multirow{2}{*}{Method}&  
			\multicolumn{4}{c}{\textsc{ALFA}}&\multicolumn{4}{c}{\textsc{Min-Max}}\cr 
			\cmidrule(lr){2-5} \cmidrule(lr){6-9}  
			& $4\%$ & $6\%$ & $8\%$ & $10\%$ & $4\%$ & $6\%$ & $8\%$ & $10\%$\cr
			\midrule  
			\textsc{DBSCAN} & $\bf{0.72}$ & $\bf{0.79}$ & $\bf{0.84}$ & $\bf{0.86}$ & $\bf{0.72}$ & $\bf{0.79}$ & $\bf{0.84}$ & $\bf{0.87}$ \cr		
			\textsc{Slab} & $<0.1$ & $<0.1$ & $<0.1$ & $<0.1$ & $0.17$ & $0.22$ & $0.27$ & $0.30$ \cr
			\textsc{L$2$} & $0.34$ & $0.37$ & $0.40$ & $0.40$ & $0.67$ & $0.66$ & $0.65$ & $0.69$ \cr	
			\textsc{Loss} & $0.11$ & $0.60$ & $0.37$ & $<0.1$ & $0.14$ & $0.28$ & $0.37$ & $<0.1$ \cr	
			\textsc{Knn} & $<0.1$ & $<0.1$ & $<0.1$ & $<0.1$ & $0.17$ & $0.11$ & $<0.1$ & $<0.1$ \cr 
			\bottomrule  
		\end{tabular}  
	\end{threeparttable} 
	\label{tab-3} 
\end{table*}  

\begin{table*}[ht]  	
	\centering  
	%	\fontsize{6.5}{8}\selectfont  
	\begin{threeparttable}  
		\caption{$F_1$ scores on \textsc{Satimage} dataset. }  
		\label{tab:4}  
		\begin{tabular}{ccccccccc}  
			\toprule  
			%			\multirow{2}{*}{Method}&  
			\multicolumn{4}{c}{\textsc{ALFA}}&\multicolumn{4}{c}{\textsc{Min-Max}}\cr 
			\cmidrule(lr){2-5} \cmidrule(lr){6-9}  
			& $4\%$ & $6\%$ & $8\%$ & $10\%$ & $4\%$ & $6\%$ & $8\%$ & $10\%$\cr
			\midrule  
			\textsc{DBSCAN} & $\bf{1.00}$ & $\bf{1.00}$ & $\bf{1.00}$ & $\bf{1.00}$ & $\bf{0.95}$ & $\bf{0.97}$ & $\bf{0.98}$ & $\bf{0.98}$ \cr	
			\textsc{Slab} & $<0.1$ & $<0.1$ & $<0.1$ & $<0.1$ & $<0.1$ & $<0.1$ & $<0.1$ & $<0.1$ \cr	
			\textsc{L$2$} & $1.00$ & $1.00$ & $1.00$ & $0.98$ & $0.63$ & $0.75$ & $0.88$ & $0.85$ \cr	
			\textsc{Loss} & $0.71$ & $0.52$ & $0.45$ & $0.41$ & $0.34$ & $0.41$ & $0.52$ & $0.44$ \cr	
			\textsc{Knn} & $<0.1$ & $<0.1$ & $<0.1$ & $<0.1$ & $<0.1$ & $0.33$ & $0.52$ & $0.18$ \cr 
			\bottomrule  
		\end{tabular}  
	\end{threeparttable} 
	\label{tab-4} 
\end{table*}

\end{document}